\documentclass[twoside]{article}

\usepackage[utf8]{inputenc} 
\usepackage[T1]{fontenc}    
\usepackage{hyperref}       
\usepackage{url}            
\usepackage{booktabs}       
\usepackage{microtype}      
\usepackage{graphicx}
\usepackage{subcaption}
\usepackage{hyperref}       
\usepackage{tikz}           

\hypersetup{ 
	pdftitle={An Analysis of Causal Models Adaptation Speed},
	pdfkeywords={},
	pdfborder=0 0 0,
	pdfpagemode=UseNone,
	colorlinks=true,
	linkcolor=blue, 
	citecolor=blue, 
	filecolor=blue, 
	urlcolor=blue, 
	pdfview=FitH,
	pdfauthor={Anonymous},
}

\usepackage{math_commands}

\newcommand*{\expect}[2][]{\ensuremath{\mathbb{E}_{#1} \left[ #2 \right] }} 

\newcommand*{\ptrue}{\ensuremath{\bm{p}}}
\newcommand*{\causalvar}{\theta_\rightarrow}
\newcommand*{\antivar}{\theta_\leftarrow}
\newcommand*{\pmodel}{\ensuremath{\bm{p}_{\causalvar}}}
\newcommand*{\panti}{\ensuremath{\bm{p}_{\antivar}}}

\newcommand*{\ptransfer}{\ptrue^*} 
\newcommand*{\natgauss}{\cN_\text{nat}}
\newcommand*{\chogauss}{\cN_\text{cho}}
\newcommand*{\Dir}{\mathrm{Dir}}
\newcommand*{\uniform}{\vu}
\newcommand*{\lr}{\gamma}

\newcommand{\mm}{m}
\newcommand{\logpartition}{A}
\newcommand{\alpin}{\alpha}
\newcommand{\nn}{n}
\newcommand{\logpartitionb}{B}
\newcommand{\betin}{\beta}

\newcommand*{\smoothness}{B}

\newtheorem{proposition}{Proposition}

\usepackage[accepted]{aistats2021}

\usepackage[round]{natbib}

\begin{document}

\runningauthor{R\'emi Le Priol, Reza Babanezhad, Yoshua Bengio, Simon Lacoste-Julien}

\twocolumn[
\aistatstitle{An Analysis of the Adaptation Speed of Causal Models}
\aistatsauthor{ 
    R\'emi Le Priol 
    \And  Reza Babanezhad Harikandeh}
\aistatsaddress{ 
    Mila, Universit\'e de Montreal 
    \And  SAIT AI Lab, Montreal}
\aistatsauthor{
    Yoshua Bengio
    \And Simon Lacoste-Julien}
\aistatsaddress{
    Mila, Universit\'e de Montreal\\ Canada CIFAR AI Chair
    \And Mila, Universit\'e de Montreal\\ Canada CIFAR AI Chair \\ SAIT AI Lab, Montreal }
]

\begin{abstract}
Consider a collection of datasets generated by unknown interventions on an unknown structural causal model $G$.
Recently, \citet{bengio2019meta} conjectured that among all candidate models, $G$ is the \emph{fastest to adapt} from one dataset to another, along with promising experiments. 
Indeed, intuitively $G$ has less mechanisms to adapt, but this justification is incomplete.
Our contribution is a more thorough analysis of this hypothesis.
We investigate the adaptation speed of cause-effect SCMs.
Using convergence rates from stochastic optimization, we justify that a relevant proxy for adaptation speed is distance in parameter space after intervention.
Applying this proxy to categorical and normal cause-effect models, we show two results.
When the intervention is on the cause variable, the SCM with the correct causal direction is advantaged by a large factor.
When the intervention is on the effect variable, we characterize the relative adaptation speed. 
Surprisingly, we find situations where the anticausal model is advantaged, falsifying the initial hypothesis.
\end{abstract}

\section{\uppercase{Introduction}}

A learning agent interacting with its environment should be able to answer questions such as ``what will happen to $Y$ if I change $X$''.
Structural Causal Models (SCM) offer a formalism to answer this kind of questions \citep{pearl2009causality,peters2017elements}. 
The simplest SCM is the model $X\rightarrow Y$ where $X$ is the cause and $Y$ the effect.
Modifying $X$ will modify $Y$ but modifying $Y$ will not alter $X$.
In general, SCMs model the distribution of observations with a directed graph where edges represent \textit{independent mechanisms} \citep{janzing2010causal}.

\begin{figure}
    \centering
    \includegraphics[width=.9\columnwidth]{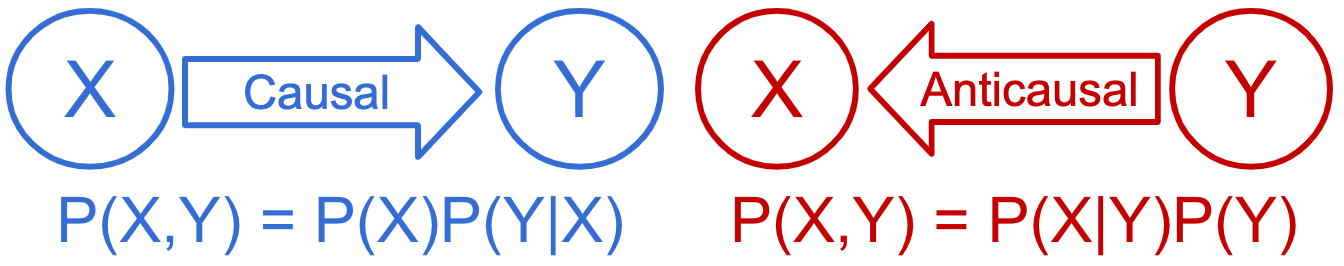}
    \caption{Two models of cause-effect data $X\rightarrow Y$ .}
    \label{fig:cause-effect}
\end{figure}

Modern machine learning methods can fail surprisingly when the test distribution differ from the training distribution \citep{rosenfeld2018elephant}. 
A recent line of work describes these distribution shifts as interventions in an underlying causal model \citep{zhang2013domain, magliacane2018domain}. 
If this description is accurate, then an agent endowed with this hypothetical causal model could handle distribution shifts by updating the few mechanisms affected by the intervention.
On contrary, an agent endowed with an incorrect model, would have to update many mechanisms.
\citet{bengio2019meta} infer that the causal agent will be the fastest to adapt to distribution shifts.
Conversely, they use the speed of adaptation to unknown interventions as a criterion to learn the true causal model, showing promising empirical results on cause-effect models.
Yet they lack a theoretical argument to connect interventions and fast adaptation. Thus we raise the question: 

\textit{Do causal models adapt faster than non-causal models to distribution shifts induced by interventions?}

\paragraph{Contributions.}
We theoretically and empirically answer this question for cause-effect SCMs with categorical variables, and partially  for multivariate normal distributions. 
\begin{itemize}
    \item For both settings, we use stochastic optimization convergence rates to show that the adaptation speed mostly depends on the distance in parameter space between the initialization (before intervention) and the optimum (after intervention).
    \item For categorical variables, we fully characterize this distance. We show that the causal model is faster by a large factor when the intervention is on the cause. 
    \item When the intervention is on the effect, we surprisingly find settings where the anticausal model is systematically faster. As appealing as the fastest-to-adapt hypothesis may sound, it does not hold in every situations.
\end{itemize}

\begin{figure*}
    \centering
    \includegraphics[width=\textwidth]{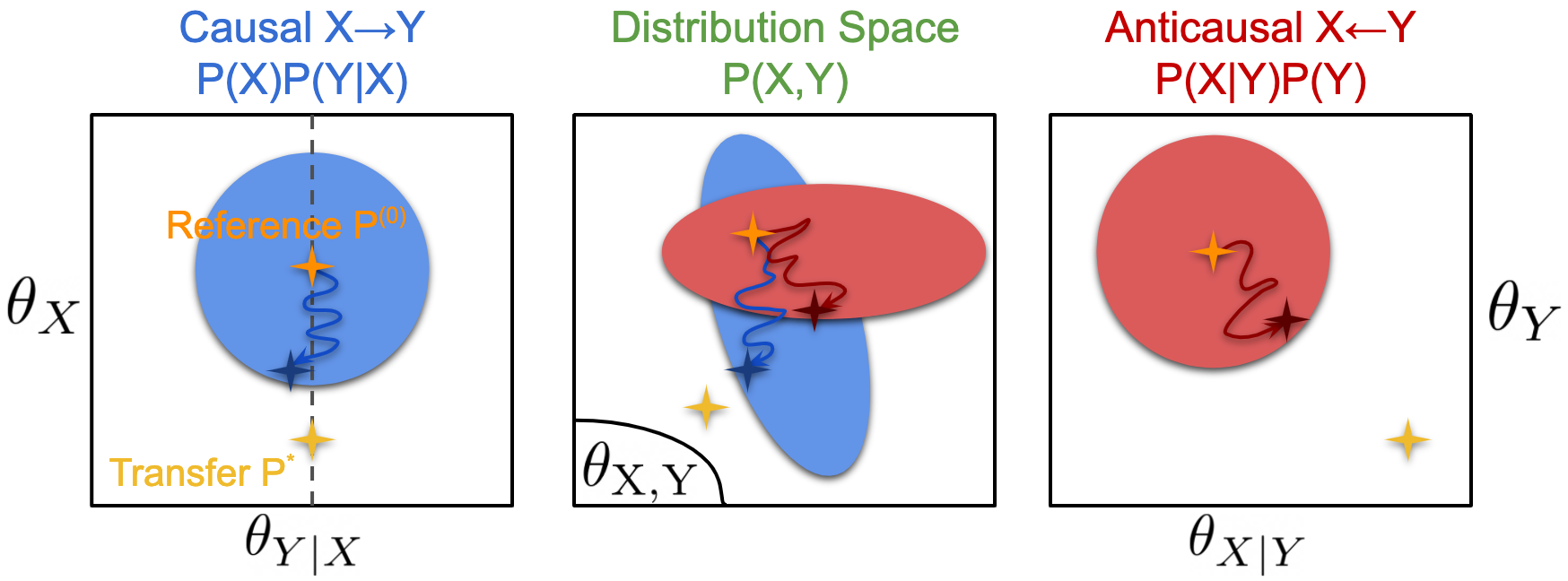}
    \caption{
    \textbf{Intuition behind fast adaptation.}
    An intervention on $X$ turns the reference distribution $\ptrue^{(0)}$ into  a transfer distribution $\ptrue^*$.
    The causal model (blue) only has to adapt $\theta_X$, whereas the anticausal  model (red) has to adapt both its mechanisms. 
    After adaptation, the causal model ends up the closest from the transfer in terms of KL, as visible in the abstract distribution space. 
    Blue and red balls represent
    the \textit{proximity prior} induced by taking a few steps of SGD from the reference in each parameter space.
    Convergence rate analysis reveals that they are spherical functions of the parameter distance,
    but they get mapped to non-trivial shapes in distribution space -- ellipses in this sketch.
    }
    \label{fig:geometry}
\end{figure*}

\section{\uppercase{Related work}}

Causal relationships are asymmetric.
These asymmetries are often visible in observations, so that one can identify which is cause and which is effect under relevant assumptions \citep{mooij2016distinguishing}. 
A common assumption is to constrain the set of functional dependencies between cause and effect. 
By contrast, in our work, we focus on two families of distributions which are notoriously unidentifiable from observational data: categorical and linear normal variables \citep[Ch.4]{peters2017elements}.
With data coming from a generic directed acyclic graph (DAG), we can only hope to discover the Markov equivalence class of this DAG \citep{verma1991equivalence}.
Many methods seek to achieve this goal, whether constraint-based such as the PC algorithm \citep{spirtes2000causation} or score-based methods using greedy search \citep{chickering2002optimal} or more recently continuous optimization \citep{zheng2018dags,lachapelle2019gradient}.
However to discover the exact graph, we need access to interventional data.

Inferring  causal links from interventions or experiments is the foundation of science. 
Inferring causal links from unknown interventions is a much harder and less principled problem.
\citet{tian2001causal} first studied this setting, proposing a constraint based method to infer the interventional equivalence class  from a sequence of interventions.
Then \citet{eaton2007exact} proposed an exact Bayesian approach.
More recently, \citet{squires2019permutation,ke2019learning}  proposed score based algorithms, improving in scalability and alleviating parametric assumptions.
From a machine learning perspective, we are concerned with the predictive power that this structure will give us when faced with new data.

Distribution shifts are a common problem in machine learning, as well as in causal statistics \citep{zhang2013domain, pearl2014external}.
\citet{scholkopf2012causal} first brought up the idea of \emph{invariance} to tackle this problem.
Following up on this idea, \citet{peters2016causal} designed an algorithm able to identify robust causal features from heterogeneous data.
This work has set a fruitful line of research for robust machine learning \citep{heinze2018invariant, heinze2018causal, rothenhausler2019causal, arjovsky2019invariant}.
In a way, fast adaptation is the complementary idea of invariance: if most mechanisms are kept invariant, then only a few have to adapt. 
\citet{scholkopf2019causality} shed light on these approaches and the broader scope of causality research for machine learning.

\section{\uppercase{Background}}

In this section, we review the formalism of \citet{bengio2019meta} on observations, interventions, models and adaptation.

\paragraph{Reference and Transfer Distributions.}
We assume perfect knowledge of a reference distribution $\ptrue$ over the pair $(X,Y)$ sampled from an SCM $X\rightarrow Y$. This distribution is the object of interventions, which results in new \emph{transfer} distributions $\ptransfer$.
If the \textit{intervention is on the cause}, $X$ is sampled from a different marginal, then $Y$ is sampled from the reference conditional
\begin{equation}
    \ptransfer(x,y) = \ptransfer(x) \ptrue(y|x) \; .
\end{equation}
If the \textit{intervention is on the effect}, $X$ is sampled from the reference marginal, then $Y$ is sampled from another marginal independently of $X$
\begin{equation}
    \ptransfer(x,y) = \ptrue(x) \ptransfer(y) \; .
\end{equation}
For each transfer distribution, we observe a few samples.

\paragraph{Models.} 
We parametrize two generative models of~$(X, Y)$ (Fig.~\ref{fig:cause-effect}):
\begin{align}
    \pmodel (x, y) = \ptrue_{\theta_X}(x)  \ptrue_{\theta_{Y|X}}(y|x)  \quad 
    &\textit{-- causal} \label{eq:causal_eq}\\
    \panti (x,y) = \ptrue_{\theta_Y}(y)  \ptrue_{\theta_{X|Y}}(x|y) \quad 
    &\textit{-- anticausal}  \label{eq:anti_causal_eq}\; .
\end{align}
For each model, we call mechanisms the marginal and conditional models.
Each mechanisms has its own set of parameters, e.g. $\theta_X$ and $\theta_{Y|X}$.
In the following we will use $\theta$ to denote interchangeably $\causalvar$ and $\antivar$.

\paragraph{Adaptation.}
Both models are initialized to fit perfectly the reference distribution $\ptrue_{\causalvar^{(0)}} = \ptrue_{\antivar^{(0)}} = \ptrue$.
They observe fresh samples from $\ptransfer$ one by one and update their parameters $\causalvar$ and $\antivar$  to maximize the log-likelihood with a step of stochastic gradient (SGD).
Thanks to the separate parameters, the causal model log-likelihood loss decomposes as
\begin{align}
    &  \cL_\text{causal}(\causalvar)
    = \expect[(X,Y) \sim \ptransfer]{-\log \pmodel(X,Y)}\nonumber
    \\ = & \expect[\ptransfer]{-\log \ptrue_{\theta_X}(X)} + \expect[\ptransfer]{-\log \ptrue_{\theta_{Y|X}}(Y|X)}
    \label{eq:causal_loss}
\end{align}
When $\ptransfer$ comes from an intervention, \citet{bengio2019meta} observe that the causal model is often faster to adapt than the anticausal model.
Intuitively, this is because the causal model has to adapt only the mechanism which was modified by the intervention.
On the other hand, the anticausal model has to adapt both its mechanisms. In Figure~\ref{fig:geometry}, we compare these different scenarios and the concept of adaptation figuratively.
While appealing, \emph{this reasoning is not rigorous}, as sample complexity bounds of SGD typically do not depend on the number of parameters to update \citep[Th. 6.2 \& 6.3]{bubeck2015convex}.
In the next section, we formalize and understand this phenomenon in the light of convergence rates of stochastic optimization methods.

\paragraph{Distribution Families.} 
We study two of the simplest sub-families of the exponential family \citep{wainwright2008graphical}: categorical and linear normal variables.
Their negative log-likelihood is a convex function of their natural parameter.
These families are interesting because the direction is not identifiable from observational data  \citep[Ch.4]{peters2017elements}
-- e.g. $\pmodel$ and $\panti$ can model the same set of distributions -- 
which makes them challenging for causal discovery.

\uppercase{An optimization perspective}
\label{sec:opt_alg}
One way to formalize adaptation speed is to characterize it via the convergence speed of the stochastic optimization procedure. An appealing aspect of stochastic optimization algorithms such as SGD (when only using fresh samples and running it on the true loss we care about) is that they come with convergence rate guarantees on the \emph{population risk} in machine learning, thus giving us direct sample complexity results to obtain a specific generalization error.
The convergence rate is an \emph{upper bound} on the expected suboptimality after a given number of iterations.
While these rates are about worst case performance and might also be loose, fortunately, for convex optimization, they tend to correspond well to actual empirical performance~\citep{nesterov2004Intro}.
We can thus use the convergence bounds as theoretical proxy for the convergence speed.
In our experiments, we also verify empirically that the bounds correlate well with the observed convergence speed.

Here we provide a classical convergence rate on the expected suboptimality with Average Stochastic Gradient Descent (ASGD) under convexity and bounded gradient assumptions. 
We re-derive this rate in Appendix~\ref{apdx:asgd_rate} for completeness.
This rate applies to log-likelihood maximization for categorical random variables (details in~\ref{apdx:categorical_optimization}). 
Since the target distribution is part of the model family, the log-likelihood suboptimality is equal to the KL-divergence  -- e.g. $\cL(\theta) - \cL(\theta^*) = \KL(\ptrue^* || \ptrue_\theta)$.

\paragraph{ASGD.} 
Assume $\forall \theta, x, \|\nabla \log \ptrue_\theta(x) \| \leq \smoothness$.
After $T$ iterations of SGD on~\eqref{eq:causal_loss},
\begin{equation}
    \theta^{(t+1)} = \theta^{(t)} + \gamma \nabla \log \ptrue_{\theta^{(t)}}(X_t,Y_t)
\end{equation}
with learning rate $\lr := \frac{c}{\sqrt{T}}$,
starting from $\theta^{(0)}$,
the average parameter's $\bar \theta^{(T)} = \inv{T} \sum_{t=0}^{T-1} \theta^{(t)}$ suboptimality  is upper bounded by
\begin{align}
    \expect{\KL(\ptrue^*||\ptrue_{\bar \theta^{(T)}})}
    \leq \frac{
    c^{-1} \|\theta^{(0)} - \theta^* \|^2
     + c \smoothness^2}{2 \sqrt{T}}
    \label{eq:sgd_rate}
\end{align}
where the expectation is taken over the sampling of $T-1$ training points $X_t,Y_t$ and $\theta^*$ is the closest solution to $\theta^{(0)}$ in the solution set
$\argmin_\theta \cL(\theta)$. 
For categorical models, $\smoothness=2$ (see~\ref{apdx:categorical_optimization}).
Consequently, for a fixed $T$ and with small enough $c$, the convergence upper bounds for causal and anticausal models  differ mainly by $\delta := \|\theta^{(0)} - \theta^* \|^2$.

The bounded gradient assumption of~\eqref{eq:sgd_rate} does not apply to the log-likelihood of normal variables. 
In Section~\ref{ssec:proximal_gradient}, we provide an algorithm along with a convergence rate~\eqref{eq:mirror_prox_rate} that do apply to this case.
Overall both bounds~\eqref{eq:sgd_rate} and~\eqref{eq:mirror_prox_rate} carry the same message which can be summarized by:
\begin{center}
\textit{The adaptation speed is dominated by \\ the initial distance}
\end{center}
\begin{align}
    \delta_\text{causal} 
    &= \norm{\causalvar^{(0)} - \causalvar^*}^2\\
    \delta_\text{anticausal} 
    &= \norm{\antivar^{(0)} - \antivar^*}^2 \; .
\end{align}

\paragraph{Other optimization methods.}
\citet[Theorem 1]{yang2016unified} provides a unified convergence rate for stochastic heavy ball and Nesterov methods that is similar to~\eqref{eq:sgd_rate}, where the initial distance is the main difference between causal and anticausal models.
Consequently, our theoretical analysis holds for a larger class of algorithms than ASGD. More generally, it applies to any stochastic optimization method whose sample complexity depends on parameter distance.
 
\section{\uppercase{Categorical variables}}
\label{sec:categorical}

In this section, both cause and effect come from categorical distribution. 
We provide theoretical bounds on $\delta_\text{causal}$ and $\delta_\text{anticausal}$. 
We consider different scenarios to generate reference and transfer data and explain the consequences of each scenario.  

\begin{figure}
    \centering
    \includegraphics[width=.8\columnwidth]{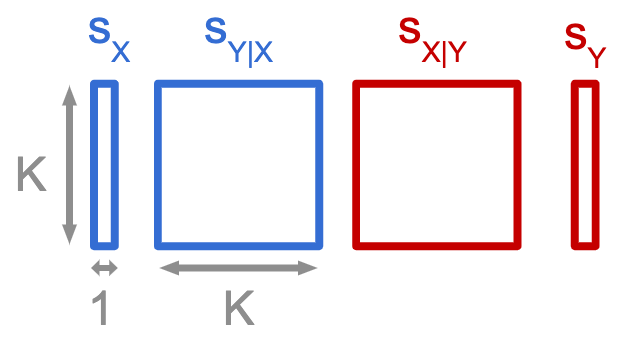}
    \caption{Parametrization of categorical models}
    \label{fig:categorical_models}
\end{figure}

\subsection{Definitions}
Cause $X$ and effect $Y$ are now two categorical variables taking values in $\{1, \dots, K\}$.
Categorical variables are an exponential family with mean parameters $\vp\in\simplex_K$ the probability vector, and with natural parameter $\vs\in\real^K$ -- the logits or score parameters such that $\evp_z = \frac{e^{\evs_z}}{\sum_{z'} e^{\evs_{z'}}}$.
The causal model has parameters $\vs_X := (\evs_x)_{x=1\dots K}$ 
and $\vs_{Y|X} := (\evs_{y|x})_{x,y=1\dots K}$. 
We gather the causal parameters in the variable $\causalvar = (\vs_X, \vs_{Y|X})$ and the anticausal parameters in $\antivar = (\vs_Y, \vs_{X|Y})$ (Fig.~\ref{fig:categorical_models}).
The loss~\eqref{eq:causal_loss} becomes
\begin{align}
    &\cL_\text{causal}(\causalvar) 
    = \expect[(X,Y)\sim\ptransfer]{-\log \pmodel(X, Y)}\\
    &= \expect[\ptransfer]{- \evs_X + \log\sum_{x} e^{s_{x}} 
    - \evs_{Y|X} + \log\sum_y e^{s_{y|X}}}
    \; . \nonumber
\end{align}
Each mechanism's stochastic loss is the sum of a linear function and a softmax function.
The softmax function is convex and 1-Lipschitz, so we can apply rate~\eqref{eq:sgd_rate}. 
To be self-contained, we include details in Appendix~\ref{apdx:categorical_optimization}.

\subsection{Distance after Intervention}
\label{ssec:categorical_intervention_cause}
In this section, we prove that interventions on the cause advantage the causal model by a factor $K$, and we describe when interventions on the effect will advantage one model over another.

\paragraph{Intervention on cause $X$, } $\vs_X \assign \vs^*_X$.
The causal  conditional $\vs_{Y|X}$ is left unchanged,
but the effect marginal $\vs_Y$ is modified in a non-trivial way. 
Consequently the initial distances are 
\begin{align}
    &\delta_\text{causal} 
    = \|\vs_X - \vs^*_X\|^2 
    \\
    &\delta_\text{anticausal} 
    = \|\vs_Y - \vs^*_Y\|^2 
    + \sum_y \|\vs_{X|y} - \vs^*_{X|y}\|^2 \; .
\end{align}
The causal model has to update $K$ parameters, whereas the anticausal model has to adapt $K^2 + K$ parameters. Therefore the causal model seems to be advantaged by a factor $K$. The following proposition -- proved in  Appendix~\ref{apdx:categorical_analysis} -- shows that this is reflected by $\ell_2$ distances.
\begin{proposition}
\label{prop:categorical_cause}
When the intervention happens on the cause,
\begin{equation}
    \delta_\text{anticausal} \geq K \delta_\text{causal} \; .
\end{equation}
\end{proposition}

\paragraph{Intervention on effect $Y$, } $\forall x, \vs_{Y|x} \assign \vs^*_Y$.
Cause and effect become independent. 
The causal model is advantaged only if the intervention $\vs^*_Y$ is close enough from the previous marginal, as formalized by the following proposition:
\begin{proposition}
\label{prop:categorical_effect}
When the intervention happens on the effect
\begin{align}
    \Delta 
    &:= \delta_\text{causal} - \delta_\text{anticausal} \nonumber \\
    &= (K-1) \left(\norm{\vs^*_Y - \vc}^2 - R^2 \right)
\end{align}
where $R^2\approx K\widehat\Var_X[\log\sum_ye^{\evs_{y|X}}]$ 
and $\vc = \frac{\left(\sum_x \vs_{Y|x}\right) - \vs_Y }{K-1}$.
\end{proposition}
See Figure~\ref{fig:effect_sphere} for an illustration and Appendix~\ref{apdx:categorical_analysis_effect} for the exact formula of $R$ and the proof. 
When the intervention $\vs_Y^*$ is close enough to $\vc$, which depends on the reference, the causal model is advantaged. 
If $\vs_Y^*$ is far from $\vc$ or if $R$ is small then the anticausal model is likely to be advantaged.

\begin{figure}
    \centering
    \includegraphics[width=.9\columnwidth]{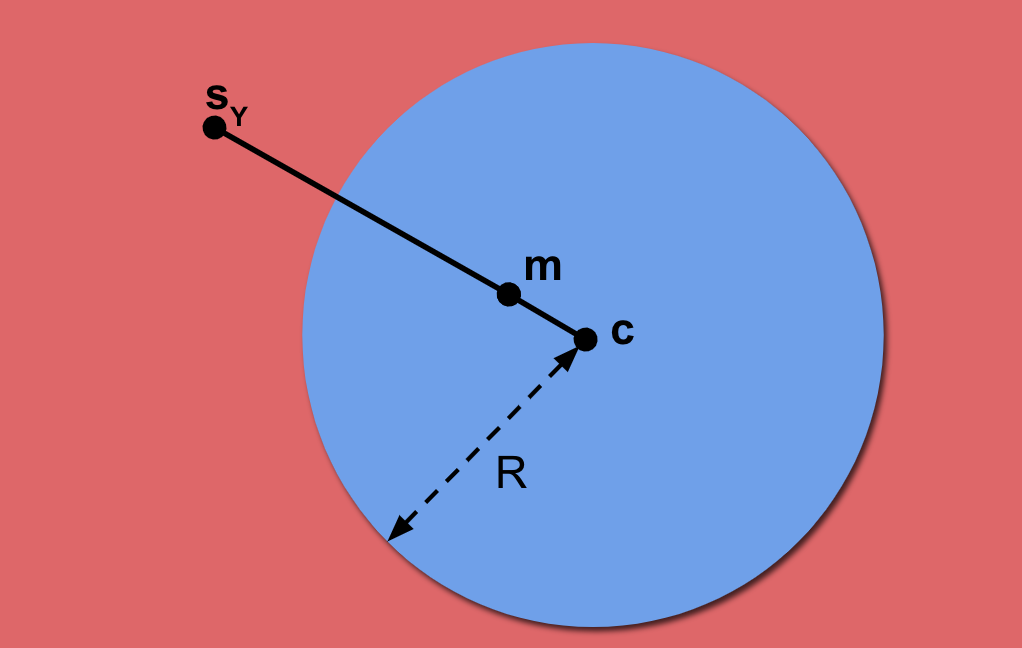}
    \caption{
    \textbf{Illustration of Proposition~\ref{prop:categorical_effect}.}
    $\vc$ is on the line joining $\vs_Y$ and $\vm := \inv{K} \sum_x \vs_{Y|x}$. 
    When $\vs^*_Y$ is within the blue ball of radius $R$ centered at $\vc$, $\Delta \leq 0$ and the causal model is advantaged, otherwise the anticausal model is advantaged (red area). 
    This is a surprising counter-example to the adaptation-speed hypothesis.}
    \label{fig:effect_sphere}
\end{figure}

\newlength{\hcolw}
\setlength{\hcolw}{0.32\textwidth}
\newlength{\scolw}
\setlength{\scolw}{0.1\textwidth}

\begin{figure*}[t]
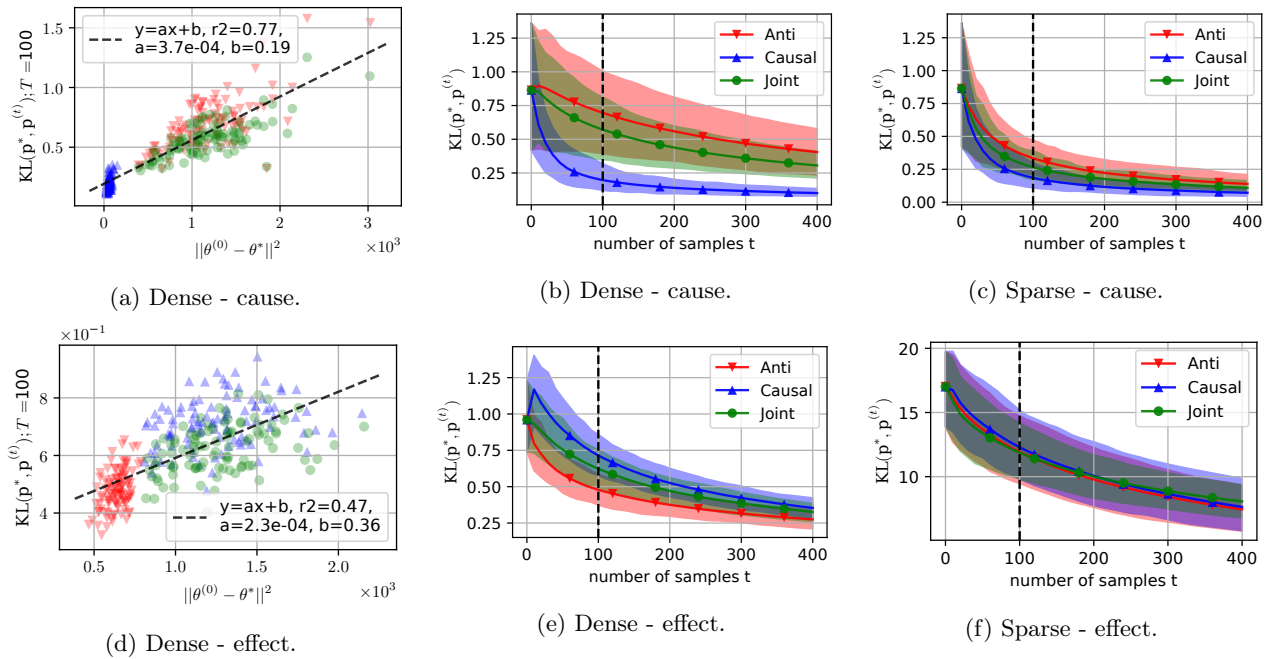

    \centering
    \foreach \intervention in {cause,effect}{
        \foreach \plot in {scatter, curves}{
            \begin{subfigure}{\hcolw}
                \centering
                \includegraphics[width=\textwidth]{img-aistats/\plot_dense_\intervention_k=20.pdf}
                \caption{Dense - \intervention.}
                \label{fig:dense_\intervention_\plot}
            \end{subfigure}
        }
        \begin{subfigure}{\hcolw}
                \centering
                \includegraphics[width=\textwidth]{img-aistats/curves_sparse_\intervention_k=20.pdf}
                \caption{Sparse - \intervention.}
                \label{fig:sparse_\intervention_curves}
        \end{subfigure}
    }
    \caption{
    \textbf{Experimental results on categorical data.}
    Each plot is captioned with the prior and the intervention considered.
    \textbf{Scatter plots} are showing the positive correlation between the KL after 100 steps of SGD and the initial parameter distance. Each point represent one of 100 synthetic pairs $(\ptrue^{(0)},\ptrue^*)$.
    \textbf{Training curves} show the average KL (solid line) and the (5,95) percentiles (shaded) over 100 runs. Remark how all models start from the same initial KL, but they converge at different speeds.}
    
    \label{fig:categorical_results}
\end{figure*}

\subsection{Simulating Reference Distributions}
\label{sec:categorical_initialization}
To evaluate the fast adaptation criterion, we are going to work on synthetic data, which raises the question : from which distribution should we sample $\ptrue = \ptrue_{\theta^{(0)}}$? 
We call this distribution \emph{prior}.
Following the independent mechanism assumption, the marginal on the cause $\vp_X$ and  the conditional of effect given cause $\vp_{Y|X}$ should not contain any information about each other. 
\paragraph{Dense Prior.}
To sample causal mechanisms, a natural choice is
\begin{align}
    \label{eq:dense_prior}
    \vp_X \sim \Dir(\ones_K) 
    \quad \text{and} \quad
    \forall x, \vp_{Y|x} \sim \Dir(\ones_K)
\end{align}
where $\Dir$ is the Dirichlet distribution and $\ones_K$ is the all-one vector of dimension $K$.
$\Dir(\ones_K)$ the uniform law over the simplex $\simplex_K$.
This prior leads to the K2 score from the Bayesian network literature \citep{cooper1991bayesian}. 
We call this choice the \emph{dense prior} by opposition to the sparse prior introduced next.
This is the choice made in \citet{bengio2019meta}, as well as \citet{chalupka2016estimating}. 
The latter work reports that distributions sampled from this prior exhibit some asymmetry between $X$ and $Y$.
In Appendix~\ref{apdx:dense_prior}, we complement their work, explaining how the effect marginal is likely to be closer from the uniform distribution than the cause marginal.
This asymmetry means that \textit{the causal direction  is identifiable from observational data}.  

\paragraph{Sparse Prior.}
To fix this issue, we study an alternative prior that is symmetric and ensures that both cause and effect marginals are sampled from a uniform prior over $\simplex_K$. We sample the causal mechanisms as follows
\begin{align}
    \label{eq:sparse_prior}
    \vp_X \sim \Dir(\ones_K) 
    \quad \text{and} \quad
    \forall x, \vp_{Y|x} \sim \Dir(\ones_K /K) \; .
\end{align}
The $\ones_K /K$ parameter means that samples will be approximately sparse, hence the name.
We show in Appendix~\ref{apdx:sparse_prior} that with this sampling scheme, the joint is sampled from a sparse Dirichlet over $\simplex_{K^2}$: $\vp_{(X, Y)} \sim \Dir(\ones_{K^2} /K)$. 
This in turns means that we can switch the roles of $X$ and $Y$ in~\eqref{eq:sparse_prior}.
The effect marginal has uniform density over the simplex.
In general, \textit{the causal direction is not identifiable from observational data.}
In Bayesian Networks literature, this is known as the Bayesian Dirichlet equivalent uniform prior \citep{heckerman1995learning}.

\subsection{Categorical Variables Experiments}
\label{ssec:categorical_experiments}

\paragraph{Goal.}
As discussed in Section \ref{sec:categorical_initialization}, the prior over the joint distribution on $(X, Y)$ is going to influence the behavior of ASGD. 
We are seeking answers to two questions:
\begin{enumerate}
    \item Is the adaptation speed positively correlated with the initial distance, as suggested by the upper bound \eqref{eq:sgd_rate} on the convergence rate of ASGD?
    \item Is there a clear difference in adaptation speed between causal and anticausal models?
\end{enumerate}

\paragraph{Data.}
We consider categorical variables with $K =20$.
For each initialization method, we sample 100 different reference joint distributions. For each of these distributions, we sample an intervention by sampling a probability vector $\vq$ \textit{uniformly} from $\simplex_K$.
If the intervention is on the cause, we plug $\vq$ instead of $\vp_X$. 
If the intervention is on the effect, we redefine $\vp_{Y|x} = \vq, \forall x$.

\paragraph{Models.}
We are comparing causal and anticausal models adaptation speed. We also report results for a model of the joint 
$\vp_{X,Y} = \mathrm{softargmax}(\vs_{X,Y})$
as a reference model. We expect its results to be in between the performance of the causal and anticausal model as it expresses no prior over the direction. 
We optimize all models with Averaged SGD. In each iteration of SGD we get one fresh sample from the transfer distribution. For each model and each setting, we tune the (constant) learning rate so as to optimize the likelihood after seeing $\frac{K^2}{4}=100$ samples, to explore the few samples regime. 
We present results in Figure~\ref{fig:categorical_results}

\paragraph{Dense prior.}
When the intervention is on the cause, the causal model is much closer from its optimum: in Fig.~\ref{fig:dense_cause_scatter} the blue cluster is on the left of the scatter plots. 
This is well correlated with faster adaptation (Fig.~\ref{fig:dense_cause_curves}). 
On the contrary, \textit{when the intervention is on the effect, the anticausal model starts closer from its optimum} and it converges faster
(Fig.~\ref{fig:dense_effect_scatter}, \ref{fig:dense_effect_curves}).
We can interpret this result in light of Proposition~\ref{prop:categorical_effect}.
In Appendix~\ref{apdx:categorical_analysis_effect}, we explain why the radius $R$ is small under the dense prior.
As a result, $\vs_Y^*$ is mostly sampled outside of the ball of radius $R$, consequently the anticausal model is advantaged. 
Overall, there is a wider gap between models in  Fig.~\ref{fig:dense_cause_curves} than in Fig.~\ref{fig:dense_effect_curves}.
Consequently, if we take a balanced average of a few interventions on the cause and a few interventions on the effect, the causal model remains faster (details in Appendix~\ref{apdx:other_categorical_results}).

\paragraph{Sparse prior.}
When the intervention is on the cause, the causal model has a slight advantage (Fig.~\ref{fig:sparse_cause_curves}). 
When the intervention is on the effect, no model has a set advantage (Fig.~\ref{fig:sparse_effect_curves}), but the sparsity induces much higher KL values, as explained in Appendix~\ref{apdx:cat_sparse_explosion}.
This KL explosion drowns the signal coming from the cause intervention, calling for further algorithmic developments -- such as inferring the intervention, as explored by \citet{ke2019learning}.

\section{\uppercase{Multivariate normal variables}}
In this section, we analyze the case of two multivariate normal variables with a linear relationship.
Cause $X$ and effect $Y$ are sampled from the causal model
\begin{align}
    \label{eq:normal_mean_model}
    & X \sim \cN(\mu_X, \Sigma_X)
   \\ 
    & Y|X \sim \cN(\mA X + \va, \Sigma_{Y|X})
\end{align}
with mean parameters $\mu_X, \va \in \real^K$ and $\Sigma_X, \mA, \Sigma_{Y|X} \in \real^{K\times K}$.
This parametrization is the most intuitive but it is unfortunately not appropriate to get convergence rates.
We are going to introduce another parametrization  
along with an algorithm and a convergence rate (Sec.~\ref{ssec:proximal_gradient}),
before providing empirical results (Sec.~\ref{ssec:normal_experiments}).

\begin{figure*}[t]
    \centering
    \hskip 0.06\textwidth
    \foreach \name in {Natural Distance, Cholesky Distance, Speed vs Cholesky distance, Learning Curves}{
        \begin{minipage}{0.22\textwidth}
        \centering
        \name
        \end{minipage}
    }
    \foreach \intervention in {cause, effect}{
        \begin{minipage}{0.06\textwidth}
        \centering
        \intervention
        \end{minipage}
        \foreach \plottype in {distnat, distcho, scatter, curves}{
            \begin{subfigure}{0.22\textwidth}
            \centering
            \includegraphics[width=\textwidth]{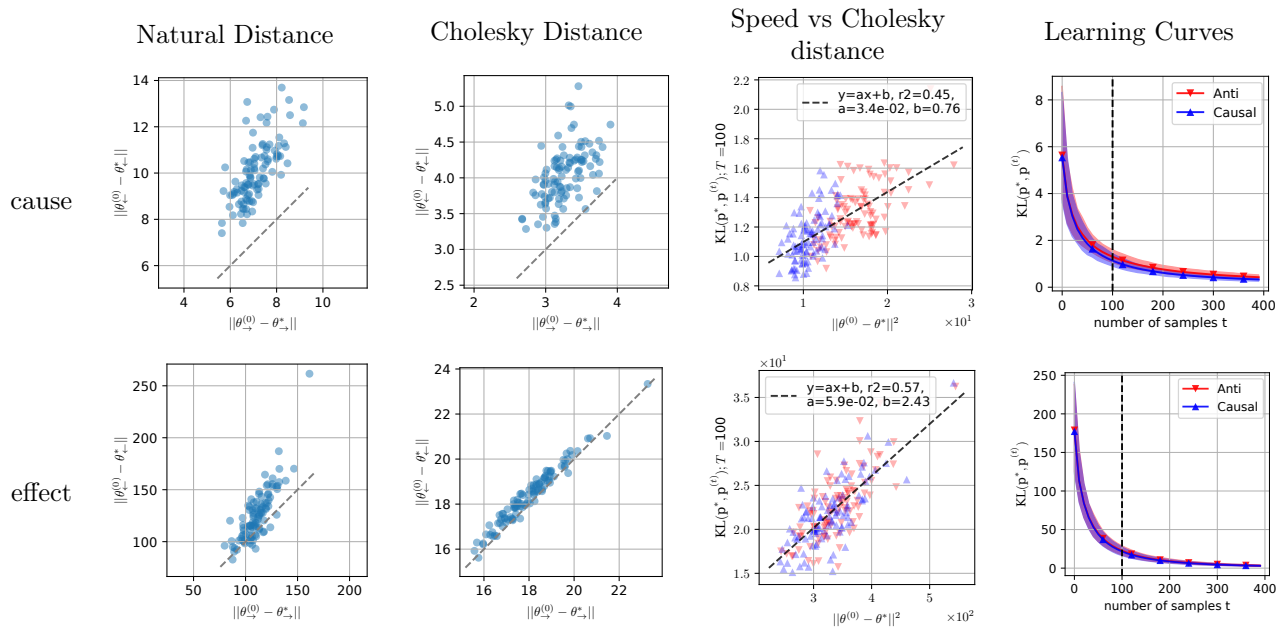}
            \end{subfigure}
        }
    }
    \caption{
        \textbf{Multivariate Normal Variables with dimension $K=10$}.
        Row 1 and 2 correspond to interventions on cause and effect respectively.
        \emph{Column 1 \& 2:} scatter plot $\delta_\text{anticausal}$ vs $\delta_\text{causal}$ respectively in natural and Cholesky parametrization. The grey diagonal is the identity line.
        We observe a natural tendency for $\delta_\text{anticausal} > \delta_\text{causal}$ (points above the grey diagonal), but this is systematically true only for the natural distance when the intervention is on the cause.
        \emph{Column 3 \& 4:} same plot as in Figure~\ref{fig:categorical_results}. Once again we observe a correlation between initial distance and optimization speed. When the intervention is on the cause, the causal model is advantaged. 
        When the intervention is on the effect, both curves overlap.
    }
    \label{fig:normal_results}
\end{figure*}

\subsection{Optimization Analysis}
\label{ssec:proximal_gradient}
The negative log-likelihood of model~\eqref{eq:normal_mean_model} is notoriously non-convex. This is problematic for convergence results. 
For simplicity, we focus in this section on the simple marginal mechanism with mean parameters $\mu, \Sigma$. We detail the full model in Appendix~\ref{apdx:normal_analysis}.
If we use the natural parameters 
$\eta=\Sigma\iinv\mu$ and $\Lambda=\Sigma\iinv$ (precision matrix), the negative log-likelihood is convex
\begin{align}
    \label{eq:natgauss_nll}
    & \expect{- \log \ptrue_{(\eta, \Lambda)}(X)} \\
    &= \half \Big(
    \expect{\Tr (X X^\top \Lambda)  - 2 X^\top\eta}
    + \eta^\top \Lambda^{-1} \eta - \log{\abs{\Lambda}} \Big) \; . \nonumber
\end{align}
This objective is composed of a pleasant stochastic linear term,
and a difficult deterministic barrier objective which goes to infinity when $\Lambda \rightarrow 0$.
This barrier is composed of a matrix inverse and a log determinant.
The assumptions of Lipschitz or gradient-Lipschitz required to get SGD convergence do not hold for the barrier.
While the empirical version of~\eqref{eq:natgauss_nll} has a close formed formula for its global minimum, quite surprisingly, gradient-based optimization of the normal likelihood is difficult to analyze. 
Convex optimization typically deals with non-smooth terms by introducing proximal operators \citep{parikh2014proximal}. 
However this barrier term is too complex to get an analytic formula for the proximal operator.
We transform it into a more convenient form by introducing $\mL$, the lower triangular Cholesky factor of the precision matrix $\Lambda = \mL\mL^T$, and $\zeta = \mL^{-1} \eta = \mL^{\top} \mu$. 
Then~\eqref{eq:natgauss_nll} simplifies into
\begin{align}
    \label{eq:chogauss_nll}
    & \expect{ {-\log \ptrue_{(\zeta, \mL)}(X)} } \\
    & ={\half \expect{{\norm{\mL^\top X - \zeta}^2}} }
    {- \sum_i \log \mL_{i,i}} \; . \nonumber
\end{align}
We will refer to $(\zeta, \mL)$  as \emph{Cholesky parameters}.
This objective is more suitable to gradient based optimization with a simple proximal operator, as detailed in the next section. 
We provide all details about the causal model in Appendix~\ref{apdx:normal_analysis}.

\paragraph{Stochastic Proximal Gradient Algorithm}
We want to minimize the sum of a stochastic convex smooth function $f_X(\theta):=\half \norm{\mL^\top X - \zeta}^2 $ and convex non-smooth regularizer $g(\theta) = - \sum_i \log \mL_{i,i}$.
This is exactly the goal of the stochastic proximal gradient \citep{duchi2010composite} update
\begin{align}
    \label{eq:stochastic_prox_gradient}
    \!\!\! \theta_{t+1} = \argmin_\theta 
    g(\theta) + \inv{2\lr_t}\norm{
    \theta_t - \lr_t\nabla f_{X_t}(\theta_t) - \theta}^2
\end{align}
where~$\lr_t$ is the step-size and $X_t$ is randomly sampled.
For objective~\eqref{eq:chogauss_nll}, the proximal gradient update has a closed form solution that amounts to updating all parameters with the stochastic gradient of the quadratic term, then updating the diagonal elements of $\mL$ with the mapping $x \mapsto \half (x + \sqrt{x^2 + 4\lr})$, thus ensuring that they remain strictly positive (details in Appendix~\ref{apdx:normal_updates}).

\paragraph{Convergence Rate.}
We assume that stochastic gradients are almost-surely $\smoothness$-Lipschitz.
$\smoothness$ is known as the smoothness constant.
We show in Appendix~\ref{apdx:normal_conv_rate} that running the stochastic proximal gradient algorithm
with
step size $\lr_t = \frac{\lr}{3\smoothness \sqrt{T}}$ where $\lr\leq 1$, for $T$ iterations guarantees
\begin{multline}
    \expect{\KL(\ptrue^*||\ptrue_{\bar \theta^{(T)}}) }  \\
    \leq \frac{3\smoothness \|\theta^{(0)} - \theta^*\|^2}{\lr \sqrt{T}}
    + \frac{ \KL(p^*||p_{\theta^{(0)}})}{T} \; .
    \label{eq:mirror_prox_rate}
\end{multline}

\paragraph{Analysis.}
The term $\nicefrac{ KL(p^*||p_{\theta^{(0)}})}{T}$ is equal for causal and anticausal models because we assume $\pmodel^{(0)} = \panti^{(0)}$.  
For normal variables, $\smoothness$ depends only on the data and is a priori equal for both models (Appendix~\ref{apdx:normal_constants}).
Similarly to \eqref{eq:sgd_rate}, both models' rates differ mainly by $\delta=\|\theta^{(0)} - \theta^*\|^2$.

When the intervention is on the cause, we prove in Appendix~\ref{apdx:normal_interventions} that the anticausal model is farther away from its optimum in the natural parametrization
\begin{align}
    \delta_\text{anticausal}^\text{natural} \geq  \delta_\text{causal}^\text{natural} \; . 
\end{align}
Unfortunately, in the Cholesky parametrization (Fig.~\ref{fig:normal_results}, 2nd column), or when the intervention is on the effect (Fig.~\ref{fig:normal_results}, bottom row),we observe empirically that there is no such hard guarantee, although the causal distance tends to be smaller than the anticausal distance.

\subsection{Experiments}
\label{ssec:normal_experiments}

Similarly to categorical variables, we need to decide on a prior over reference and transfer distributions. This choice is informed by two criteria.
First the independent mechanism principle which states that we should sample $\theta_X$ independently of $\theta_{Y|X}$.
Second we want $\theta_Y$ to have approximately the same distribution as $\theta_X$ 
-- e.g. we want the distribution to be approximately symmetric so that we cannot identify the direction from observational data. 
These considerations lead us to a flavor of normal-Wishart prior \citep{geiger2002parameter} described in Appendix~\ref{apdx:normal_experiments}.

We sample $100$ random joint distributions from this prior, and for each distribution we sample a random intervention on the cause, and a random intervention on the effect. We then run the stochastic proximal gradient on objective~\eqref{eq:chogauss_nll}. We report results in Figure~\ref{fig:normal_results}.
Similarly to the categorical case, when the intervention is on the cause, the causal model is advantaged by a slight margin (upper right figure). When the intervention is on the effect both models are learning at the same speed (bottom right figure).

\section*{Conclusion}
We provided a first theoretical analysis of the adaptation speed in two-variables cause-effect SCMs under localized interventions for categorical and normal data.
Convergence guarantees for stochastic optimization on the true population log-likelihood indicates that the adaptation speed is related to the distance between initial point and optimum in parameter space.
We verified this correlation empirically. 
We proved analytically that this distance is lower for the causal model than for the anticausal model when the intervention is on the cause variable.
This explains a surprising phenomenon: while both models start with the same suboptimality, one learns faster than the other.
When the intervention is on the effect variable, we highlighted examples showing that either model can be advantaged.
This observation challenges the intuition that the causal model should be the fastest to adapt, and it raises new questions for the approach of~\citet{bengio2019meta}, such as: are there practical situations where the fastest-to-adapt heuristic is useful ? 
On a more theoretical note, is it possible to characterize the adaptation speed behavior for more general families of distributions? 

\subsubsection*{Acknowledgments}
We thank S\'ebastien Lachapelle and Tristan Deleu for insightful discussions and useful feedback that led to the main results of this paper. 
We thank Waïss Azizian for his help in debugging proofs.
This work was partially supported by the Canada CIFAR AI Chair Program. 
Simon Lacoste-Julien is a CIFAR Associate Fellow and Yoshua Bengio is a Program Co-Director in the Learning in Machines \& Brains program.

\bibliographystyle{plainnat}
\bibliography{references}

\clearpage
\newpage
\appendix

\newcommand{\tet}{\theta_t}
\newcommand{\ttt}{\theta_{t+1}}
\newcommand{\ts}{\theta^{*}}

\onecolumn

\section{SOURCE CODE}
Source code for all experiments is hosted at 
\url{https://github.com/remilepriol/causal-adaptation-speed}.

\section{\uppercase{Categorical optimization}}

In this section we prove a convergence rate of ASGD that applies to the categorical loss, and we show that the constants involved in this rate are the same for both causal and anticausal models.

\subsection{Convergence of ASGD with Fixed Step-Size}
\label{apdx:asgd_rate}
Here we derive a classical convergence rate of Average SGD.
This result is standard ; we include it to be self-contained.
The objective is 
\begin{align}
\min_{\theta} F(\theta) = \expect[i]{f(\theta,i)} \; .
\end{align}

\begin{theorem}
If each $f_i(\theta) = f(\theta,i)$ has bounded gradient $B$,
then after $T$ steps of SGD with step-size $\gamma=\frac{c}{\sqrt{T}}$,
starting from $\theta_0$,
the expected sub-optimality verifies
\begin{align}
   \expect{F(\bar{\theta}_T) - F(\ts)} \leq \frac{1}{2c\sqrt{T}}\|\theta_0- \ts\|^2 + \frac{ c B^2}{2\sqrt{T}}
\end{align}
where $\bar{\theta}_T = \frac{1}{T} \sum_t \tet$.
\end{theorem}

\begin{proof}
First we relate the $\ell^2$ distance to optimum at step $t+1$ with the one at step $t$ :
\begin{align*}
    &\| \ttt - \ts \|^2 \\
    &= \| \tet - \ts \|^2 - 2 \gamma \left < f'_i(\tet) , \tet-\ts \right> + \gamma^2\|f_i'(\tet)\|^2 \\
    &\leq \| \tet - \ts \|^2 - 2 \gamma \left < f'_i(\tet) , \tet-\ts \right> + \gamma^2B^2 \; .
\end{align*}
By convexity of $f_i$ and rearranging the terms we get 
\begin{align*}
2\gamma( f_i(\tet) - f_i(\ts) ) & \leq 2 \gamma \left < f'_i(\tet) , \tet-\ts \right> \\
& \leq  \| \tet - \ts \|^2 - \| \ttt - \ts \|^2+ \gamma^2B^2.
\end{align*}
Now we take the expectation, sum up both sides for $T$ iterations and divide by $2T\gamma$ to get
\begin{align*}
    &\frac{1}{T} \sum_{i=1}^T \expect{F(\tet) - F(\ts)} \\
    & \leq \frac{1}{2\gamma T}\left(\expect{\|\theta_0- \ts\|^2}
    - \expect{\|\theta_{T+1}- \ts\|^2}\right)
    + \frac{\gamma B^2}{2}\\
    & \leq \frac{1}{2\gamma T}\|\theta_0- \ts\|^2 + \frac{\gamma B^2}{2}
\end{align*}
Finally, we apply Jensen inequality to $F$ in $\bar{\theta}_T = \frac{1}{T} \sum_t \tet$
to get the final result.
\end{proof}

\subsection{Categorical Loss Properties}
\label{apdx:categorical_optimization}

We are now going to verify that assumptions of the rate \eqref{eq:sgd_rate} apply to the negative log-likelihood loss for the categorical distribution.
This loss is standard and it's properties are well-known, but we review them here to be self-contained.

Each mechanism has the same form of negative log-likelihood, with the same kind of stochastic gradients.
The total loss is a sum over mechanisms, and the total stochastic gradient is a concatenation of each mechanisms stochastic gradient.
To apply rate~\ref{eq:sgd_rate}, we can either apply it separately on each mechanism, either apply to the whole. Both path lead to the same result.
In the end, we simply have to check that this loss is convex, and has bounded gradients for all $z$.
The random functions coming from sampling $Z$ are 
\begin{equation}
    f_z(\vs) = -s_z + \log(\sum_{z'} e^{s_{z'}} )
\end{equation}
This function is the softmax -- or logsumexp -- function minus a  stochastic linear term. 

\paragraph{Convexity}
We are going to show that it is convex but not strongly convex because it becomes flat for large score values.
Its derivative is
\begin{equation}
    \nabla f_z(\vs) = -\ve_z + \vp \; .
\end{equation}
where $\ve_z$ is the $z$-th canonical basis element  and $\vp$ is the output of the softargmax function taken on $\vs_Z$.
The Hessian is the same for every $z$.
\begin{equation}
    \nabla^2 f_z(\vs) = \text{diag}(\vp) - \vp \vp^\top \; .
\end{equation}
We observe that for any vector $\vv$,
\begin{align}
    \label{eq:hessian2variance}
    \vv^\top \nabla^2 f_z(\vs) \vv &= \sum_z \evp_z\evv_z^2 - \left(\sum_z \evp_z \evv_z \right)^2 \nonumber \\
    &= \Var_{Z\sim \vp}[v_Z] \geq 0
\end{align}
which means that the logsumexp is convex. 
When $\evs_0$ tends toward positive infinity and the other components remain constant, $\vp$ tends toward a Dirac on the 0-th component.
Then \eqref{eq:hessian2variance} is 0 for all $\vv$, so the logsumexp is not strongly convex.


\paragraph{Bounded Gradients.}
The gradient norm is
\begin{align}
    \label{eq:cat_grd_uppr_bnd}
    \|\nabla f_z(\vs) \|
    &= \|\vp -\ve_z\| \\ \; .
\end{align}
This norm is maximized for $\vp=\ve_{z'}, \forall z' \neq z$.
The maximum is equal to $\sqrt{2}$.
If there are $d$ independent mechanisms (for $d$ variables in the graph), then the total stochastic gradient which is a concatenation of all gradients has a norm bounded by $B=\sqrt{2d}$. 
In our case of cause-effect models, $d=2$ and and the gradients are bounded by $B=2$, or in other words, all the $f_z$  are 2-Lipschitz.

This bound is the same for causal and anticausal models. It depends on the part of space where $\vp$ is going to live. Assuming that it is going to live in most of the space for both directed models, both loss will have the same Lipschitz constants in practice.

Thanks to these properties, the sample complexity of $\pmodel$ and $\panti$ are bounded by \eqref{eq:sgd_rate}. 
The difference in adaptation speed between causal and anticausal models is characterized by the distance in parameter space.



\section{\uppercase{Categorical analysis}}
\label{apdx:categorical_analysis}

In this section, we prove relationships between parameter distances induced by interventions between the causal and anticausal models. 
First we prove two useful lemmas. 
Then we establish that the causal model dominates the anticausal model by a factor $K$ when the intervention is on the cause. 
Finally we show that no model has a set advantage when the intervention bears on the effect.

The logits or scores $\vs$ live in $\real^K$.
They have one additional degree of freedom compared to the probability $\vp$. 
More specifically, the softargmax is invariant by translations along the vector $\ones = (1,\dots, 1)$.
In other words, all scores $\{ \vs + \lambda \ones | \forall \lambda \in \real \}$ are equivalent. 
Scores which move by following the gradient of this loss will remain in the same affine hyperplane orthogonal to $\ones$.
To ensure that the distances we measure are meaningful, we project all logits in the hyperplane such that $\sum_z \evs_z = 0$, by subtracting their mean.
\begin{definition}[Mean-zero score]
\label{def:mean_zero}
A score vector $\vs$ is mean-zero iff $\sum_z \evs_z = 0$.
\end{definition}

\subsection{Switching Direction}
In this section we are going to prove a few useful results relating cause and anticausal models.
We know the causal parameters $X \rightarrow Y$, and we want to find the corresponding $X \leftarrow Y$ model,
e.g. express $\vs_Y, \vs_{X|Y}$ as a function of $\vs_X, \vs_{Y|X}$ .
This will help us to find a relationship between $\delta_\text{causal}$ and $\delta_\text{anticausal}$.
We first need to define a few useful variables
\begin{definition}[Average conditional score vectors]
\label{def:average_conditional_logit}
For any $x$ or $y$, define
\begin{align}
    \mm(y):= \frac{1}{K} \sum_x \evs_{y|x}, \quad
    \nn(x):= \frac{1}{K} \sum_y \evs_{x|y} \; .
\end{align}
\end{definition}
\begin{definition}[Conditional log-partition function]
    \begin{align}
        \logpartition(x) = \log \sum_y e^{\evs_{y|x}}
    \end{align}
\end{definition}
With these variables, we can express the reverse conditional score from the causal parameters.
\begin{lemma}[Anticausal conditional score]
\label{lem:categorical_reverse}
Let $\evs_{y},\evs_{x}$ be marginal scores, 
and $\evs_{y|x},$ $\evs_{x|y}$ be conditional scores. Then
\begin{equation}
    \label{eq:bayes_scores}
    \boxed{\evs_{x|y} = \evs_x + (\evs_{y|x} - \mm(y)) - (\logpartition(x) - \alpin)} \; ,
\end{equation}
where $\alpin = \frac{1}{K} \sum_x \logpartition(x)$.
\end{lemma}

\begin{proof}
Let's apply Bayes rule to find the conditional probability mass function
\begin{align*}
    \ptrue(x|y) 
    &\propto \ptrue(y|x) \ptrue(x) \\
    &\propto \exp \left(\evs_{y|x} - \logpartition(x)
    + \evs_x\right)
\end{align*}
where $\logpartition(x)$ is the log-partition function of $\ptrue(y|x)$.
Taking the logarithm, 
\begin{equation}
    \label{eq:un_condx}
    \evs_{x|y} = \evs_x + \evs_{y|x} - \logpartition(x) + C(y)
\end{equation}
where $C(y)$ is a constant defined such that $\sum_x \vs_{x|y} = 0$ (see Definition~\ref{def:mean_zero}).
We take the sum of~\eqref{eq:un_condx} over $x$ to find
\begin{align*}
    & &\sum_x \evs_x& + &\sum_x \evs_{y|x}& - &\sum_x\logpartition(x)& + &K C(y)\\
    &= &0& + &K\mm(y)& + &K\alpin& +  &K C(y)
\end{align*}
which simplifies into
\begin{align*}
    C(y) = - \mm(y) - \alpin
\end{align*}
We plug this in~\eqref{eq:un_condx} to conclude the proof.
\end{proof}

We conclude this section with an identity showing that conditional logits are equally close from their averages in both directions.
\begin{lemma}
\label{lem:categorical_identity}
For any $x$ and $y$ we have 
\begin{equation}
    \label{eq:scores_identity}
    \boxed{\evs_{x|y} - \nn(x) = \evs_{y|x} - \mm(y)} \;.
\end{equation}
where $\nn(x):= \frac{1}{K} \sum_y \evs_{x|y}$ and $\mm(y):= \frac{1}{K} \sum_x \evs_{y|x}$. 
\end{lemma}

\begin{proof}
We apply Lemma~\ref{lem:categorical_reverse} with the roles of $X$ and $Y$ inverted to express $\vs_{y|x}$ as a function of the anti  causal parameters
\begin{align}
\label{eq:bayes_scores_anti}
    \evs_{y|x} &= \evs_y + (\evs_{x|y} - \nn(x)) - (\logpartitionb(y) - \betin) \; ,
\end{align}
where
$\logpartitionb(y) = \log \sum_x e^{\evs_{y|x}}$ 
and $\betin = \frac{1}{K} \sum_y \logpartitionb(y)$.
We can add \eqref{eq:bayes_scores} and \eqref{eq:bayes_scores_anti} 
to get rid of the conditional scores 
\begin{align*}
    \evs_x - \nn(x) - \logpartition(x) + \alpin
    = - (\evs_y - \mm(y) - \logpartitionb(y) + \betin), 
\end{align*}
for all $x,y$.
The left hand side is constant in $y$ whereas the right hand side is constant in $x$. 
Thus both sides are constants with respect to both $x$ and $y$. 
In particular they are equal to their average
\begin{align*}
    \forall x, 
    \evs_x - \nn(x) - \logpartition(x) + \alpin
    &= \inv{K} \sum_{x'} (\evs_{x'} - \nn(x'))\\
    &-\inv{K} \sum_{x'} \logpartition(x') + \alpin \\
    &= 0 - 0 - \alpin + \alpin \\
    &= 0 \; .
\end{align*}
We plug this equality into \eqref{eq:bayes_scores} to prove the lemma.
\end{proof}

\subsection{Intervention on Cause}
\label{apdx:categorical_analysis_cause}
In this section, we analyze the relationship between $\delta_\text{causal}$ and $\delta_\text{anticausal}$ after an intervention on the cause. 

\paragraph{Proposition~\ref{prop:categorical_cause}.}
\begin{itshape}
{If an intervention happens on the cause $X$ then we have}
\begin{equation}
    \boxed{
    \delta_\text{anticausal} \geq K \delta_\text{causal}} \; ,
\end{equation}
where $\delta_\text{causal} = \|\vs_X - \vs^*_X\|^2$, and $\delta_\text{anticausal} = \|\vs_Y - \vs^*_Y\|^2  + \sum_y \|\vs_{X|y} - \vs^*_{X|y}\|^2$
\end{itshape}

\begin{proof}
Given that $\evs^*_{y|x} = \evs_{y|x}, \vm^* = \vm, \logpartition^*=\logpartition$ and $\alpin^*=\alpin$,  Lemma~\ref{lem:categorical_reverse} tells us that the anticausal conditional $\vs^*_{X|Y}$ verifies
\begin{align*}
    \evs^*_{x|y} - \evs^*_x 
    = (\evs_{y|x} - \mm(y)) - (\logpartition(x) - \alpin)
    = \evs_{x|y} - \evs_x \\
    \implies
    \evs_{x|y} - \evs^*_{x|y}  
    = \evs_x  -\evs^*_x  \; .
\end{align*}
The distance between models before and after intervention are
\begin{align*}
    \delta_\text{causal}
    &= \|\vs_X - \vs^*_X\|^2 \\
    \delta_\text{anticausal}
    & = \|\vs_Y - \vs^*_Y\|^2  + \sum_y \|\vs_{X|y} - \vs^*_{X|y}\|^2 \\
    & \geq 0 +  \sum_y \sum_x ( \evs_x - \evs^*_{x})^2 = K \|\vs_X - \vs^*_X\|^2 \;.
\end{align*}
In conclusion,
\begin{equation}
    \delta_\text{anticausal} \geq K \delta_\text{causal} \; .
\end{equation}
\end{proof}

\subsection{Intervention on Effect}
\label{apdx:categorical_analysis_effect}
The following proposition shows that when the intervention is on the effect, the causal model is advantaged only when the new effect marginal $\vs^*_Y$ is close enough from the previous marginal. 

\paragraph{Proposition~\ref{prop:categorical_effect}}
\begin{itshape}
 When an intervention happens on the effect
\begin{align}
    \label{eq:categorical_effect_distdiff}
    \Delta 
    :&= \delta_\text{causal} - \delta_\text{anticausal}\\
    &= (K-1) \left(\norm{\vs^*_Y - \vc}^2 - R^2\right)
\end{align}
where the score vector $\vc$ and the scalar $R$ are defined as
\begin{align}
    \vc =& \frac{K\vm - \vs_Y }{K-1} \\
    (K-1) R^2 =& K  \|\vn - \vs_X \|^2  + (K-1) \norm{\vc}^2 \\
    & + \norm{\vs_Y}^2 - K\norm{\vm}^2 \nonumber
\end{align}
with $\vm$ and $\vn$ as in Definition~\ref{def:average_conditional_logit}.
\end{itshape}

We illustrate the relationship between $\vm$, $\vc$, $\vs_Y$  and $R$ in Figure~\ref{fig:categorical_sketch_effect}. 

\begin{figure}
    \centering
    \begin{subfigure}{0.45\textwidth}
    \centering
    \includegraphics[width=.9\columnwidth]{img-icml/effect_sphere.png}
    \caption{
    $\vm$ is a convex combination of $\vc$ and $\vs_Y$. 
    The blue bubble is the sub-level set 0 of $\Delta$.
    It is a circle of radius $R$ centered at $\vc$.
    Within this circle, $\delta_\text{causal} \leq \delta_\text{anticausal}$  the causal model is advantaged. 
    Outside this circle, $\delta_\text{causal} \geq \delta_\text{anticausal}$ the anticausal model is advantaged.}
    \label{fig:categorical_sketch_effect}
    \end{subfigure}
    \quad \quad
    \begin{subfigure}{0.45\textwidth}
    \centering
    \includegraphics[width=\columnwidth]{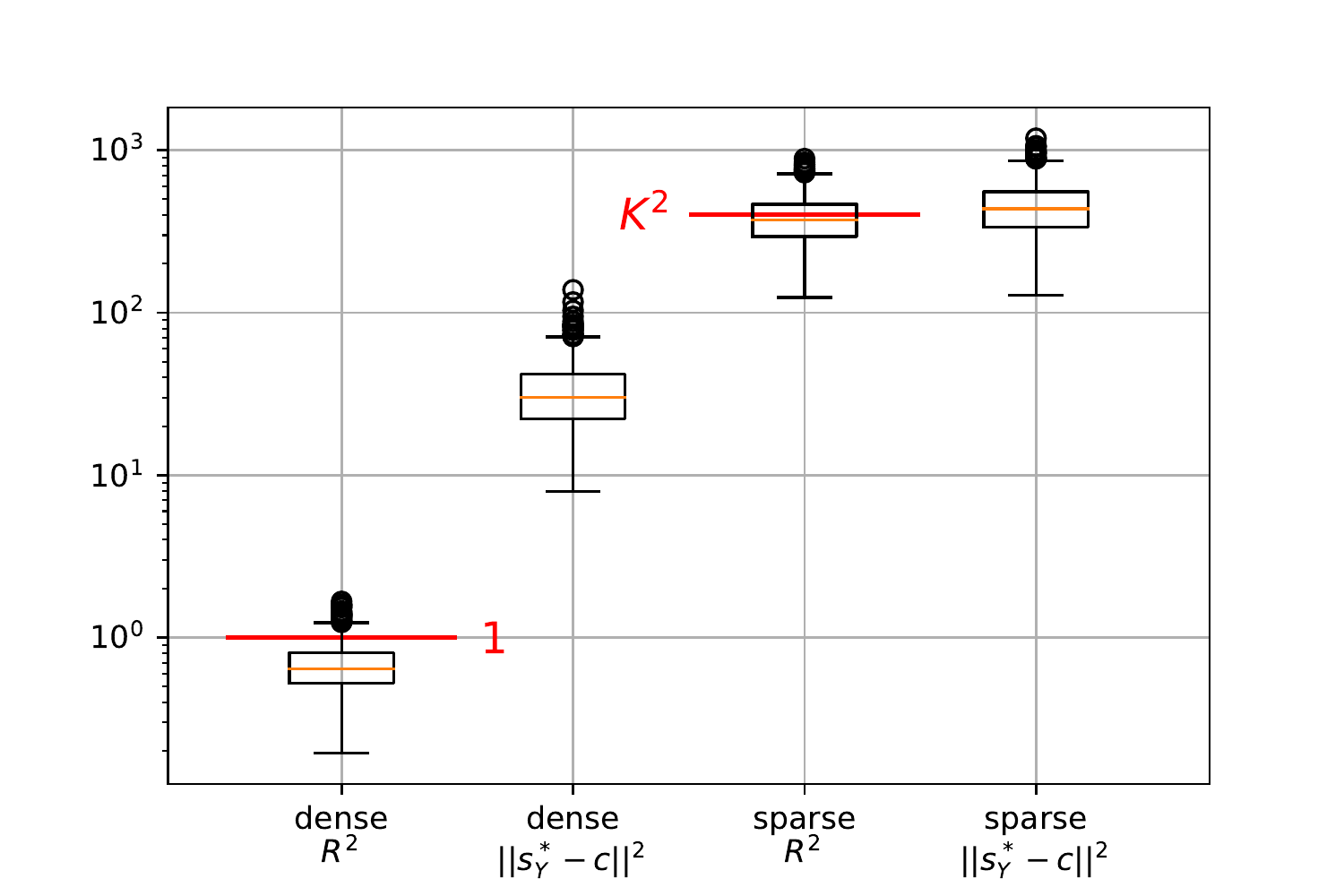}
    \caption{
    Box plots for the radius $R^2$ and deviations $\|\vs_Y^* - \vc\|^2$ for $K=20$ with the dense prior (left) and the sparse prior (right). 
    The y-axis is logarithmic.
    Red lines show analytical estimates for the expected radius.
    }
    \label{fig:radius_boxplot}
    \end{subfigure}
    \caption{Behavior of the hyper-sphere within which the causal model is advantaged, as presented in Proposition~\ref{prop:categorical_effect}.}
    \label{fig:radius_illustrations}
\end{figure}

\begin{proof}
First we expand the causal distance with a bias variance decomposition
\begin{align}
    \delta_\text{causal} 
    &= \sum_x \|\vs_{Y|x} - \vs^*_Y \|^2 \nonumber\\
    &= \sum_{x,y} (\evs_{y|x} - \mm(y) + \mm(y)- \evs^*_y )^2 \nonumber\\
    &= \textcolor{blue}{\sum_{x,y} (\evs_{y|x} - \mm(y))^2 } 
    +K \sum_y ( \mm(y)- \evs^*_y )^2 
    \label{eq:categorical_effect_causal} \nonumber\\
    & \quad + 2\sum_y (\mm(y)- \evs^*_y ) \underbrace{\sum_x (\evs_{y|x} - \mm(y))}_{=0} \; . 
\end{align}
Given that $\vs^*_{X|y} = \vs_{X}$, we can decompose the anticausal distance similarly
\begin{align}
    \delta_\text{anticausal} 
    &= \|\vs_Y - \vs^*_Y \|^2 + \sum_y \| \vs_{X|y} - \vs^*_{X|y}\|^2 \nonumber\\
    &= \|\vs_Y - \vs^*_Y \|^2 + \sum_{x,y}  (\evs_{x|y} - \nn(x) + \nn(x) - \evs_x )^2\nonumber \\
    &= \|\vs_Y - \vs^*_Y \|^2 
    + \textcolor{blue}{\sum_{x,y}  (\evs_{x|y} - \nn(x))^2}\nonumber\\ 
    &+  K \sum_x (\nn(x) - \evs_x )^2 
    \label{eq:categorical_effect_anticausal}\; .
\end{align}
Thanks to Lemma~\ref{lem:categorical_identity}, the variance of conditional score vectors (in blue) in \eqref{eq:categorical_effect_causal} and \eqref{eq:categorical_effect_anticausal} are equal
\begin{align*}
 \textcolor{blue}{\sum_{x,y}  (\evs_{x|y} - \nn(x))^2 
 = \sum_{x,y} (\evs_{y|x} - \mm(y))^2} \; .
\end{align*}
What remains in the difference is the quadratic form
\begin{align*}
    \Delta 
    &= \delta_\text{causal} - \delta_\text{anticausal} \\
    &=  K \|\vm - \vs^*_Y\|^2  
    - K\|\vn - \vs_X \|^2
    - \| \vs_Y - \vs^*_Y \|^2 \; ,
\end{align*}
which we can expand to highlight the role of $\vs^*_Y$ as
\begin{align*}
    \Delta
    =& (K-1) \|\vs^*_Y\|^2
    - 2 \langle \vs^*_Y, K\vm - \vs_Y \rangle \\
    & + K\|\vm\|^2  - \|\vs_Y\|^2 
    - K\|\vn - \vs_X\|^2 \|^2 \\
    =& (K-1) \|\vs^*_Y - \vc\|^2
    - (K-1) \|\vc\|^2 \\
    & + K\|\vm\|^2  - \|\vs_Y\|^2 
    - K\|\vn - \vs_X\|^2 
\end{align*}
where $\vc$ appears as a non-convex interpolation of $\vm$ and $\vs_Y$, 
$\vc = \frac{K\vm - \vs_Y }{K-1}$. 
Define $R^2$ to conclude the proof.
\end{proof}

\paragraph{Empirical estimates of the radius.}
We report values of $R^2$ and $\norm{\vs_Y^* - \vc}^2$ observed for the dense and sparse priors in Figure~\ref{fig:radius_boxplot}.
For dense prior radii are much smaller than  deviations, whereas for the sparse prior they have similar magnitude. 
This explains why the anticausal model systematically adapts faster when the intervention is on the effect and the prior is dense. 
We also observe that radii (and deviations)  are much greater for the sparse prior than for the dense prior.
In the following paragraph we provide some clues  to explain this behaviour.

As illustrated by Figure~\ref{fig:categorical_sketch_effect}, $\vm$ is a convex combination of $\vc$ and $\vs_Y$ : $\vm = \frac{(K-1)\vc + \vs_Y}{K}$ so by convexity of $\|.\|^2$,
\begin{align*}
    &\frac{K-1}{K} \|\vc\|^2 + \inv{K}\|\vs_Y\|^2  \geq \|\vm\|^2 \\
    \implies &(K-1) \|\vc\|^2 + \|\vs_Y\|^2  -K \|\vm\|^2  \geq 0 \\
    \implies &R^2 \geq \|\vn - \vs_X \|^2 \; .
\end{align*}
As $K$ grows larger, this inequality will get closer and closer to an equality. Indeed, $\vm$ will get closer and closer to $\vc$ and we will end up with 
\begin{align*}
    \frac{K-1}{K} \|\vc\|^2 + \inv{K}\|\vs_Y\|^2  - \|\vm\|^2 \ll \|\vn - \vs_X \|^2 \approx R^2 \; .
\end{align*}
Before proceeding, let us prove a simple proposition that is a direct consequence of Lemma~\ref{lem:categorical_reverse}.
\begin{proposition}
\label{prop:marginal_and_average_conditional}
The squared distance between marginal score and reverse average conditional is equal to the empirical variance of the conditional log-partition function
\begin{equation}
    \norm{\vs_X - \vn}^2 = K \widehat\Var_X[\logpartition(X)] \; .
\end{equation}
\end{proposition}
\begin{proof}
Taking the average of Lemma~\ref{lem:categorical_reverse} over $y$ yields
\begin{align*}
    \inv{K}\sum_y \evs_{x|y} 
    &= \evs_x - (\logpartition(x) - \alpin) + \inv{K}\sum_y (\evs_{y|x} - \mm(y))  \\
    \nn(x) 
    &= \evs_x - (\logpartition(x) - \alpin)\; ,
\end{align*}
where we used the definition of $\nn(x)$ and the mean-zero scores. 
Reordering terms gives
\begin{equation}
    \evs_x - \nn(x) =  \logpartition(x) - \alpin \;
\end{equation}
Recall that $\alpin$ is the average of $\logpartition(x)$ over $x$. 
Squaring this equation and summing over $x$ concludes the proof.
\end{proof}

Using this proposition we get that 
\begin{align*}
    R^2 &\approx \norm{\vs_X - \vn}^2 \\
    &= K \widehat\Var_X[\logpartition(X)] \\
    & = K \widehat\Var_X[\log\sum_y e^{\evs_{y|X}}]
    \; .
\end{align*}
Conditional scores $\vs_{y|x}$ are taking much greater values with much higher variance under the sparse prior than under the dense prior.
To be clear, we sample independently pseudo-scores $\tilde \evs_{y|x}$ from exp-gamma laws, and we subtract their mean to ensure that they sum to 0 (Definition~\ref{def:mean_zero}). 
$\evs_{y|x} = \tilde \evs_{y|x} - \inv{K} \sum_{y'} \tilde \evs_{y'|x} $.
This means that
\begin{align*}
    \log\sum_y e^{\evs_{y|X}} = \log\sum_y e^{\tilde \evs_{y|X}} 
    - \inv{K} \sum_{y} \tilde \evs_{y|x}
\end{align*}
If we make the approximation that the logsumexp term and the average term are independent then
\begin{align*}
    &\widehat\Var_X[\log\sum_y e^{\evs_{y|X}}] \\
    &\approx \Var_{\evs_{y|x}}[\log\sum_y e^{\evs_{y|x}}] \\
    &\approx \Var_{\tilde\evs_{y|x}}[\log\sum_y e^{\tilde \evs_{y|X}}]
    + \inv{K} \Var_{\tilde\evs_{y|x}}[\sum_{y} \tilde \evs_{y|x}] \\
    & = \psi^{(1)}(K\lambda) +  \inv{K}\psi^{(1)}(\lambda)
\end{align*}
where this last step uses the formula for the variance of an exp-gamma variable twice. 
The variance of an exponential gamma with shape parameter $\lambda$ is $\psi^{(1)}(\lambda)$ where $\psi^{(1)}$ is the trigamma function. 
The log-sum-exp of $K$ independent exp-gamma with scale and shape parameters $(\lambda, \zeta)$ is another exp-gamma with scale and shape parameters $(K \lambda, \zeta)$.
Finally we get the following approximation for the squared radius
\begin{align*}
    R^2 \approx K \psi^{(1)}(K\lambda) + \psi^{(1)}(\lambda) \; .
\end{align*}
The dense prior uses a shape parameter $\lambda=1$ while the sparse prior uses a shape parameter $\lambda=\inv{K}$.
We use two approximations of the trigamma function: $\psi^{(1)}(\inv{K})\approx K^2 + \pi^2/6$  and $\psi^{(1)}(K)\approx \inv{K}$ when $K\geq 10$.
\begin{align*}
    R^2_\text{dense} 
    &\approx K \psi^{(1)}(K) + \psi^{(1)}(1) \\
    &\approx 1 + \pi^2/6  = O(1)\\
    R^2_\text{sparse} 
    &\approx K \psi^{(1)}(1) + \psi^{(1)}(\inv{K}) \\
    &\approx K^2 + K \pi^2 /6 +  \pi^2 /6  = O(K^2) \; .
\end{align*}
In other words for dense prior the radius grows linearly with the dimension $K$. We report these estimates along with real data in Figure~\ref{fig:radius_boxplot}.

\paragraph{Independent Special Case.} 
if $X$ is independent of  $Y$ in the reference distribution -- e.g. $\forall x, y, \ptrue(x, y) = \ptrue(x) \ptrue(y)$ -- then $\forall x, y,$
\begin{align*}
    \evs_x &= \evs_{x|y} = \nn(x) \\
    \evs_y &= \evs_{y|x} = \mm(y)
\end{align*}
Plugging these equalities into Proposition~\ref{prop:categorical_effect} yields
\begin{align*}
    \vc &= \vm = \vs_Y 
    \quad \text{and} \quad
    R = 0 \\ 
    \implies 
    \Delta &=  (K-1) \norm{\vs^*_Y - \vs_Y}^2 
    \geq 0
\end{align*}
which means that the anticausal model is advantaged
$\delta_\text{causal} \geq \delta_\text{anticausal}$. 
This is actually predictable from a simple parameter counting argument. 
When the reference distribution is made of independent distributions, the anticausal conditional mechanism is already optimal $\evs_x = \evs_{x|y}$.
The anticausal model only has to adapt its marginal mechanism $\vs_Y$ of size $K$.
On contrary, the causal model only has to adapt its conditional mechanism $\evs_{y|x} \neq \evs_y$ of size $K^2$.
Overall the causal model has to adapt $K$ times more parameters than the anticausal model.

\subsection{Other Empirical Results for Cause and Effect Interventions}
\label{apdx:other_categorical_results}

In this section, we present additional results for categorical variables.
In Figure~\ref{fig:dense_results}, compared to the main text, we add what happens with the dense prior when we average learning curves (pooled) from 5 interventions on the cause and 5 interventions on the effect : on average the causal model adapts the fastest.
In Figure~\ref{fig:sparse_results}, compared to the main text we show what happens with the sparse prior, both in terms of distance (scatter plots) and in terms of pooled results. Because the intervention on the effect creates huge values of the KL, there is no set advantage for any of the models. 

\begin{figure*}
    \centering
    \foreach \intervention in {cause,effect,pooled}{
        \begin{minipage}{0.3\textwidth}
         \centering
            Dense \intervention
        \end{minipage}
    }    
        \foreach \plot in {curves, scatter}{
         \foreach \intervention in {cause,effect,pooled}{
            \begin{subfigure}{0.3\textwidth}
                \centering
                \includegraphics[width=\textwidth]{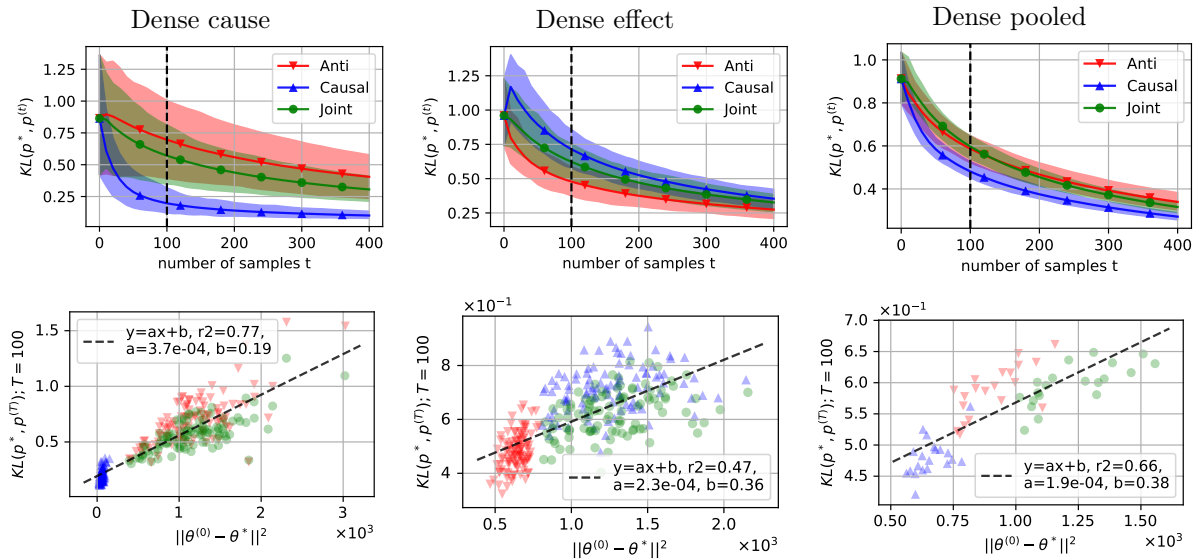}
            \end{subfigure}
        }
    }
    \caption{
        \textbf{Categorical dense prior with K=20.}
        \textit{Row 1:} training curves. Solid lines are average KL over 100 runs. We tune hyper-parameters to minimize the average KL of each model at the black vertical dashed bar ($t$=100). Shaded areas are between (5,95) quantiles. Note that \textit{all models start from the same initial KL, but they converge at different speeds.}
        \textit{Row 2:} scatter plot of the KL at $t$=100 vs. initial distance. Note that the initial distance is well correlated with the KL after 100 steps of SGD.
        \textit{Columns:} we report results for interventions on the cause on column 1, the effect on column 2, and an aggregation of both on column 3. We aggregate results by taking the average of 5 cause interventions and 5 effect interventions as one new trajectory. In total we have 20 such trajectories per model. We are reporting this result because the meta-learning criterion suggested by \citet{bengio2019meta} is akin to the average adaptation speed over a small set of interventions. 
    }
    \label{fig:dense_results}
\end{figure*}

\begin{figure*}
    \centering
    \foreach \intervention in {cause,effect,pooled}{
        \begin{minipage}{0.3\textwidth}
         \centering
            Sparse \intervention
        \end{minipage}
    }
    \foreach \plot in {curves, scatter}{
        \foreach \intervention in {cause,effect,pooled}{
            \begin{subfigure}{0.3\textwidth}
                \centering
                \includegraphics[width=\textwidth]{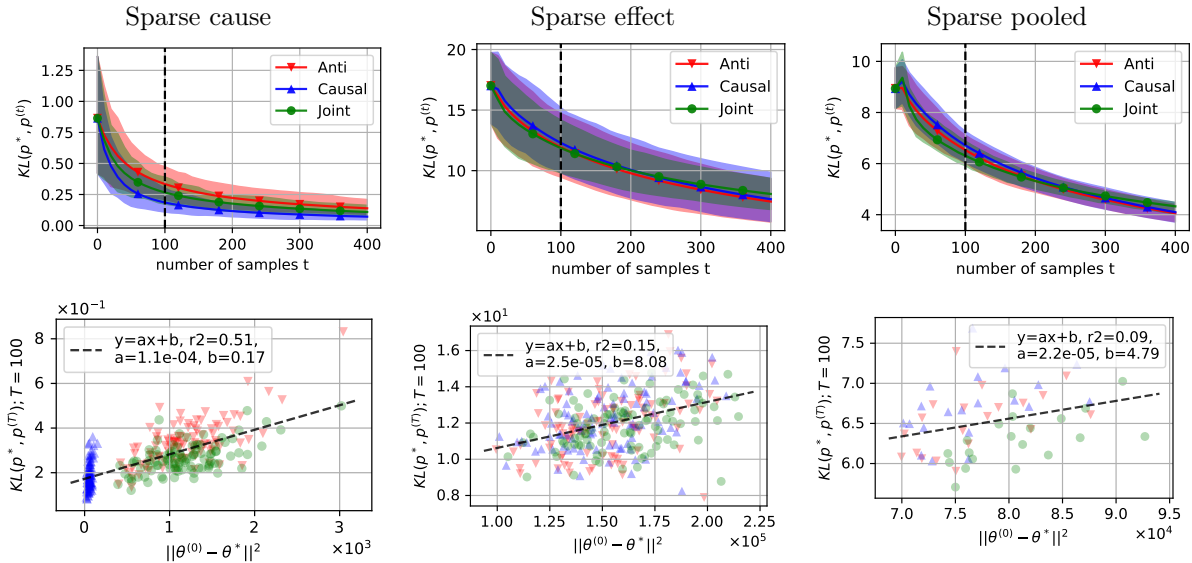}
            \end{subfigure}
        }
    }
    \caption{
        \textbf{Categorical sparse prior with K=20.}
        \emph{Column 1:} intervention on the cause. The causal model starts closer from optimum and adapts slightly faster than others.
        \emph{Column 2:} intervention on the effect. All models have the same initial distance and the same objective value. However the KL value is around 10. This is 10 times larger than when the intervention is on the cause.
        \emph{Column 3:} we take the average of 5 effect and 5 cause interventions. The effect dominates this average because it is much larger. As a result there is no signal.
    }
    \label{fig:sparse_results}
\end{figure*}

\subsection{Single Mechanism Intervention}
If only $\vs_{Y|x_0}$ changes, for some $x_0$, then from Lemma~\ref{lem:categorical_reverse} we get the following equality
\begin{align}
    \label{eq:single_mechanism}
    \delta_\text{anticausal} =
    &\frac{K-1}{K}\delta_\text{causal}
    + \norm{\vs_Y^* - \vs_Y}^2  \nonumber\\
    &+ (K-1)(\logpartition^*(x_0) - \logpartition(x_0))^2 \; .
\end{align}
The causal and anticausal distances seem to be on the same scale, with a multiplicative factor $\frac{K-1}{K} \lessapprox 1$ and a positive additive factor.
This is interesting because the sparsity argument holds: the causal model needs to change $K$ parameters whereas the anticausal model needs to change $K^2+K$ parameters.
That means we could expect an advantage by a factor $K$ for the causal model, similarly to when the intervention is on the cause.
However \eqref{eq:single_mechanism} tells another story: without further assumptions, it seems like both distances will have the same scale.

\begin{figure}
    \centering
    \foreach \init in {dense,sparse}{
        \foreach \plot in {curves, scatter}{
        \init 
            \begin{subfigure}{.4\textwidth}
                \centering
                \includegraphics[width=\textwidth]{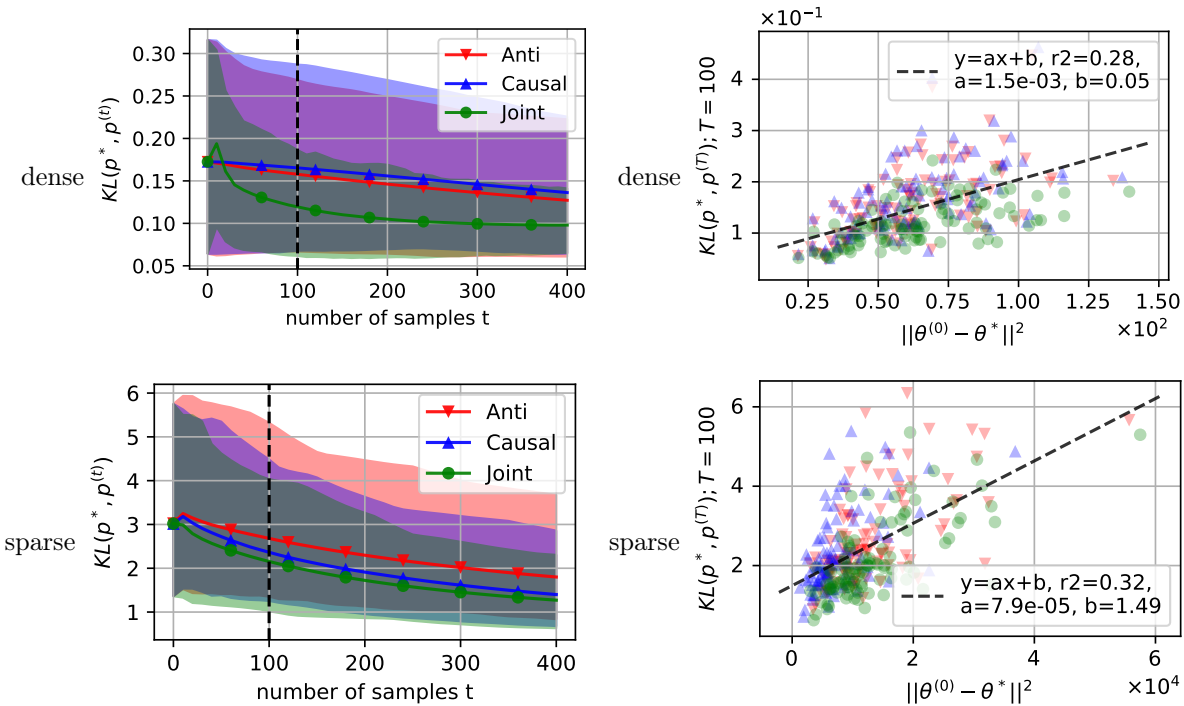}
            \end{subfigure}
        }
    }
    \caption{
        \textbf{Single mechanism intervention with K=20.}
        \emph{Two first rows:} Dense prior. The only model to be slightly advantaged is the joint model. 
        \emph{Two last rows:} Sparse prior. This time there is a slight advantage for the causal model which performs comparably to the joint model.
        Overall the optimization is hard in both settings, since we are observing only $K^2$ samples for models with $O(K^2)$ parameters. The KL barely decreases.
    }
    \label{fig:single_mechanism_intervention}
\end{figure}

\paragraph{Experiments.}
For this kind of intervention to be detectable, we need to intervene on $x_0$ such that $\ptrue(x_0)$ is quite large. 
To ensure this in our experiments, we pick $x_0 = \argmax_x \ptrue(x)$. We report results on dense and sparse priors in Figure~\ref{fig:single_mechanism_intervention}. 
We observe no significant advantage for the causal model, in spite of the parameter counting prediction.

\section{\uppercase{Categorical priors}}
In this section we study the dense and sparse prior described in the main paper.

\subsection{Causal Direction is Identifiable under the  Dense Prior}
\label{apdx:dense_prior}

\citet{chalupka2016estimating} study the prediction of causal direction from observational data under the dense prior assumption. 
The causal direction $X \rightarrow Y$ induces a certain prior over joint distributions $\pi(\vp | \rightarrow)$.  
The anticausal direction $X \leftarrow Y$ induces another one $\pi(\vp | \leftarrow)$.
The Bayes classifier is predicting $\rightarrow$ if 
\begin{align}
    \label{eq:bayes_classifier}
    \log\pi( \rightarrow | \vp) -  \log\pi( \leftarrow | \vp) > 0 \; ,
\end{align}
and $\leftarrow$ otherwise.
Under the dense prior assumption, this classifier makes an error of approximately $0.4$ for $K=2$, which decreases exponentially to $0.001$ for $K=10$.
We reproduced their setting and report the error of the optimal classifier in Table \ref{tab:chalupka_error} for varying K.
In other words the dense prior induce very asymmetric distributions which makes the causal direction identifiable. 

Is this Bayes classifier easy to estimate ? 
It turns out that under the dense prior, the criterion~\eqref{eq:bayes_classifier} can be simplified into the following criterion   
\begin{align}
    \label{eq:chalupka_criterion}
    \KL(\uniform || \ptrue_X) - \KL(\uniform || \ptrue_Y) > 0
\end{align}
where $\uniform := \ones / K$ is the uniform probability vector.
The proof is left as an exercize to the reader. 
If \eqref{eq:chalupka_criterion} is positive, Bayes predicts that the cause is $X$, otherwise $Y$ is the cause.
In words, whichever variable has the most uniform marginal is the effect.
This simple rule is optimal given the prior assumption (and if both directions are equally likely).
We can understand it from a concentration of measure perspective.
The effect marginal is written as a sum of quasi independent uniform variables
\begin{equation}
    \ptrue(y) = \sum_x \ptrue(y|x) \ptrue(x)
\end{equation}
which ends up close from the uniform vector.

\begin{table*}[]
    \small
    \centering
    \begin{tabular}{c|cccccccccccccc}
        K & 2 & 3 & 4 & 5 & 6 & 7 & 8 & 9 & 10 & 11 & 12 & 13 & 14 \\
        \toprule
        error & .4 & .2 & .1 & .07 & .03 & .01 & .005 & .002 & .0007 & .0003 & 7e-5 & 3e-5 & 4e-6 
    \end{tabular}
    \caption{Estimation of the Bayes error under the dense prior assumption for increasing categorical variables dimension K. 
    We estimated these numbers by sampling one million joint distributions for each K. We report 1 significant figure.}
    \label{tab:chalupka_error}
\end{table*}

\subsection{Joint Distribution with Sparse Prior}
\label{apdx:sparse_prior}

The following theorem shows how a Dirichlet prior over joint distributions $\evc_{x,y}= \ptrue(x,y)$ is equal to independent Dirichlet priors over marginal $\eva_x=\ptrue(x)$ and conditional $\evb_{y|x}=\ptrue(y|x)$ probability mass functions. 
By applying this theorem, we find that the sparse prior is equivalent to $\Dir(\inv{K}\ones_{K^2})$.

\begin{theorem}[Dirichlet and Factorization]
    \label{thm:dir_fact}
    Let $\vc$ be a random square matrix of dimension $K$. 
    Let's define $\va$ as the random vector obtained by summing columns of $\vc$, and $\vb$ as a copy of $\vc$ with rows normalized so that they sum to $1$. 
    $\forall (i,j) \in \{1, \dots,K\}^2$
    \begin{align}
        \eva_i  &= \sum_j \evc_{i,j} \\
        \evb_{j|i} &= \frac{\evc_{i,j}}{\eva_i} \; .
    \end{align}
    Let $\gamma$ be a positive square matrix of parameters.
    The following equivalence holds
    \begin{align}
    \label{eq:dirichlet_equivalence}
        \vc \sim \Dir(\gamma) \iff
        \begin{cases}
            &\va \sim \Dir(\sum_i \gamma_i) \\
            &\vb_{:|i} \sim \Dir(\gamma_{i,:}), \forall i \\
            &\va \indep \vb_{:|i} \indep \vb_{:|i'}, \forall i\neq i'
        \end{cases}
        \; .
    \end{align}
\end{theorem}
\begin{proof}
First let's remark that the right side of~\eqref{eq:dirichlet_equivalence} is entirely characterizing the joint distribution on $(\va,\vb)$, and that the relationship between $\vc$ and $(\va,\vb)$ is a bijection with reverse $\evc_{i,j} = \evb_{j|i} \eva_{i}$.
This means that the equivalence~\eqref{eq:dirichlet_equivalence} is an equality between distributions.
This means that we can prove the forward implication and the converse will hold automatically.

If $\vc \sim \Dir(\gamma)$, then there exist $K^2$ independent Gamma variables $\tilde \evc_{i,j} \sim \Gamma(\gamma_{i,j}, 1), \forall i,j$ such that
\begin{align}
    \vc = \frac{\tilde \vc}{S} \quad \text{where} \quad S = \sum_{i,j} \tilde \evc_{i,j} \; .
\end{align}
We know from properties of the Gamma distribution that $\vc$ is independent of $S$.
Now let's define $\tilde \eva_i := \sum_j \tilde \evc_{i,j}$.  This definition has three consequences. 
First $\tilde \eva_i$ is a sum of independent gammas, so it is a gamma with parameters $(\alpha_i := \sum_j \gamma_{i,j}, 1)$. 
Second $S= \sum_i \tilde \eva_i$.
Third $\va = \frac{\tilde \va}{S}$ is independent of $S$ and is a Dirichlet with parameter vector $\bm \alpha = \sum_j \bm \gamma_{:,j}$.

That was for the marginal. Now for the conditional, 
\begin{align}
   \evb_{j|i} = \frac{\evc_{i,j}}{\eva_i} 
   = \frac{\tilde \evc_{i,j} /S}{\tilde \eva_i /S} 
   = \frac{\tilde \evc_{i,j}}{ \tilde \eva_i}  \; .
\end{align}
Again from properties of the Gamma distribution, $\vb_{:|i} \indep \tilde \eva_i, \forall i$, and $\vb_{:|i} \sim \Dir(\gamma_{i,:})$. Each of the conditional $\vb_{:|i}$ is defined with independent gammas, so we also have the independence between conditionals.
We verified all properties of the right side of~\eqref{eq:dirichlet_equivalence}, which concludes the proof.
\end{proof}

\subsection{Categorical Sparse Prior Explosion}
\label{apdx:cat_sparse_explosion}

In this section we explain why the KL takes large values with sparse prior and effect intervention. 

On one hand the sparse prior samples probability vectors which are close from being Dirac. 
On the other hand the effect intervention creates an outer product between two samples drawn uniformly from the simplex $\Dir(\ones)$. 

For instance, for the uniform probability vector $\uniform = \frac{1}{K^2} \ones \in \simplex_{K^2}$ and an almost Dirac $\vp = (1-\eps)\ve_1 + \eps \uniform$
\begin{align*}
    \KL(\uniform || \vp)
    \in \Theta(\log(\inv{\eps}))
\end{align*}
where $\eps$ is a small value.
As we increase $K$, the sparse prior $\Dir(\ones_{K^2}/ K)$ becomes more sparse.
Conceptually, the value of $\eps$ decreases, and the value of $\KL(\uniform || \vp)$ explodes.
This is why we observe high KL values for sparse prior and effect intervention. Empirically, these values also increase with $K$.

\section{\uppercase{Normal optimization}}
\label{apdx:normal_optimization}
In this Section we adapt the stochastic composite mirror-prox algorithm to our setting of unbounded multivariate normal optimization. 
First we describe the algorithm and prove a novel convergence rate that applies to our setting. 
Then we explicit the update formulas for the normal log-likelihood loss with Cholesky parameters. 
Finally we prove that worst case constants appearing in the rate are equal for both causal and anticausal models.

\subsection{Stochastic Composite Mirror-Prox}
\label{apdx:normal_conv_rate}

We want to minimize the composite objective 
\[
F(\theta) = \expect[i]{ f(\theta,i)} + g(\theta)\; .
\]
For simplicity we denote $f(\theta,i)$ by $f_i(\theta)$ and $f(\theta) = \expect[i]{f_i(\theta)}$. 
We assume that $f_i$ is convex, 
$\nabla f_i$ is  $L$-Lipschitz
and $g$ is a convex function. 
The stochastic mirror-prox algorithm update rule at time $t$ is 
\begin{align}
    &\nu_{t} = \theta_{t+1} - \gamma_t f'_i (\theta_t)\\
    & \theta_{t+1} = \argmin_{\theta}\left\{g(\theta)+\frac{1}{\gamma_t}\cB_{h}(\theta,\nu_t)\right\}
\end{align}
where $f_i$ is sampled randomly. $B_h(x,y)=h(x)-h(y)-<h'(y),x-y>$ denotes the Bregman divergence between $x$ and $y$ induced by the convex function $h$ and we have $\|B_h(x,y)\| \geq \frac{\alpha}{2} \|x-y\|^2$. When we set $h(x)= 1/2\|x\|^2$, we recover something called the proximal stochastic gradient method \citep{duchi2009efficient}, also known as Perturbed proximal gradient algorithm \citep{atchade2017perturbed}. 
This last citation in particular has hypothesis very close to ours. 

\subsubsection{Convergence Rate}
The following Theorem is a mild modification of the Theorem 8 in~\citep{duchi2010composite}.
Our result is different in 2 ways. 
First, we remove the boundedness assumption for the Bregman divergence throughout the trajectory i.e. $\cB_{h}(\theta^*,\theta_t) \leq D$ for all $t$.
Second, we replace the $f_i$ $\smoothness$-Lipschitz continuous assumption by $\nabla f_i$ $\smoothness$-Lipschitz continuous. 
We need this last modification for the result to hold on $f_i$ quadratic.

First we prove the following lemma which is a modification of Lemma 1 in~\citep{duchi2010composite}. 
\begin{lemma}
\label{apdx:lem:mirr_prox_upbnd}
With $f$  convex and $\smoothness$-smooth, $g$  convex, and $\lr \leq \frac{\alpha}{3\smoothness}$, at iteration $t$, if we sample $i$, we have: 
\begin{align*}
\lr\left\{ f_i(\theta_t) + g(\theta_{t+1}) - F(\theta^*) \right\} \leq & \cB_{h}(\theta^*,\theta_t) - \cB_{h}(\theta^*,\theta_{t+1})\\ &+ \lr \left\{ f_i(\theta_t) - f_i(\theta_{t+1}) \right\}. 
\end{align*}
\end{lemma}
\begin{proof}
   We have the following sequence of inequality 
   \begin{align*}
        &\lr\Big( f_i(\theta_t) +  g(\theta_{t+1}) - F(\theta^*) \Big) \\  
        & \leq \lr \left < \theta_t - \theta^*, f_i'(\theta_{t}) \right> + \lr \left < \theta_{t+1} - \theta^*, \partial g_t(\theta_{t}) \right> \\ 
        & \leq \cB_{h}(\theta^*,\theta_t) - \cB_{h}(\theta^*,\theta_{t+1}) \\
        & \quad - \cB_{h}(\theta_{t+1}, \theta_t) + \lr \left < \theta_t - \theta_{t+1}, f_i'(\theta_{t}) \right>\\
        & \leq \cB_{h}(\theta^*,\theta_t) - \cB_{h}(\theta^*,\theta_{t+1}) \\
        & \quad  - \cB_{h}(\theta_{t+1}, \theta_t) + \frac{\smoothness \lr}{2} \|\theta_t - \theta_{t+1}\|^2 \\
        & \quad +\lr \Big( f_i(\theta_t) - f_i(\theta_{t+1}) \Big)\\
    \end{align*}
    where the first inequality comes from convexity of $f_i$ and $g$,the second inequality comes from Eq.(6) of lemma 1 in~~\citep{duchi2010composite}, and the third inequality comes from
    the smoothness of $f_i$. By $\lr \leq \frac{\alpha}{3L}$, the term $- \cB_{h}(\theta_{t+1},\theta_{t}) + \frac{\smoothness \lr}{2} \|\theta_t - \theta_{t+1}\|^2$ is negative : we can drop this term and get the required result. Note that $\partial g$ is a subgradient of $g$.
\end{proof}

\begin{theorem}
Given the above assumptions for $f_i$ and $g$, after $T$ iterations of the stochastic mirror prox algorithm with $\gamma = \frac{c}{\sqrt{T}},(c\leq \frac{\alpha}{3\smoothness})$, we have 
\[
    \expect{F(\bar{\theta}) - F(\theta^*) } \leq \frac{\cB_{h}(\theta^*,\theta_0)}{c \sqrt{T}} + \frac{(F(\theta_0) - F(\theta^*))}{T} \; .
\]
\end{theorem}

\begin{proof}
The proof is similar to the proof of the Theorem 8 in~\citep{duchi2010composite} with some modifications. Take the expectation of lemma~\ref{apdx:lem:mirr_prox_upbnd} with respect to the samples $(i_u)_{u\leq t}$
\begin{align*}
    &\expect{\lr\Big(f(\tet) + g(\ttt) -F(\ts) \Big)} \\ 
    &\leq  \expect{\cB_{h}(\ts,\tet) - \cB_{h}(\ts,\ttt)
    + \lr\Big( f(\theta_t) - f(\theta_{t+1})\Big)}
\end{align*}
where $\tet$ and $\ttt$ are random variable that depends on the samples.
Sum up both side for $T$ iterations: 
\begin{align*}
    &\lr \sum_{t=0}^T \expect{f(\tet) + g(\ttt) -F(\ts)} \\ 
    & \leq \cB_{h}(\ts,\theta_0) 
    - \expect{\cB_{h}(\ts,\theta_{T+1})} \\
    & + \lr \Big(f(\theta_{0}) - \expect{f(\theta_{T+1})}\Big) 
\end{align*}
By adding $\lr \expect{ g(\theta_{0}) - g(\ttt)} $ to both sides of the above inequality we get: 
\begin{align*}
    & \lr \sum_{t=1}^T \expect{F(\tet) -F(\ts)} \\
    & \leq \cB_{h}(\ts,\theta_0) 
    - \expect{\cB_{h}(\ts,\theta_{T+1})} \\
    & + \lr \Big(F(\theta_{0}) - \expect{F(\theta_{T+1})}\Big) \\
    & \leq \cB_{h}(\ts,\theta_0) 
    + \lr \Big(F(\theta_{0}) - \expect{F(\ts)}\Big)
\end{align*}
where the last inequality is due to non-negativity of Bregman divergence and optimality of $\ts$. 
Divide both sides by  $\lr T = c \sqrt{T}$ and use Jensen inequality on $F$ to conclude the proof.  
\end{proof}

 \subsection{Normal Model Updates}
 \label{apdx:normal_updates}
The objective function at hand is: 

\begin{align*}
    &F(L,\zeta) = f(L,\zeta) + g(L)\\
    &f(L,\zeta) = \frac{1}{2n}\sum_{i=1}^{n} \|L^T x_i - \zeta\|^2\\
    & g(L) = - \ln(|L|). 
\end{align*}
Now the update rule for the $\zeta$ given that $g$ is independent of $\zeta$ and we sample mini-batch $B$ of size $m$:
\[
    \zeta_{t+1} = (1 - \gamma)\zeta_t 
    + \gamma L_t^T \Big(\frac{1}{m} \sum_{i \in B} x_i\Big)
\]

For the $L$ the gradient update gives 
\[
    L_{t+\frac{1}{2}} = L_t - \gamma \frac{1}{m} \sum_{i \in B} (x_i x_i^T L_t - x_i \zeta^T_t) \; . 
\]
Since $L$ is lower triangular,  $g(L) = \ln(|L|) = \sum_{i=1}^d \log L_{i,i}$ and the proximal operator only applies to diagonal elements of $L$ 
-- e.g. when $i\neq j$ $[L_{t+1}]_{(i,j)} = [L_{t+\frac{1}{2}}]_{(i,j)}$. Otherwise we have to compute: 
\[
    [L_{t+1}]_{(i,i)} 
    =\argmin_{L_{ii}} \Big\{ 
    - \ln(L_{ii}) 
    + \frac{1}{2\gamma}{\| L_{ii} - [L_{t+\frac{1}{2}}]_{(i,i)}\|^2}\Big\}.
\]
Therefore the update rule for the $[L_{t+1}]_{(i,i)}$ is: 
\[
    [L_{t+1}]_{(i,i)} 
    = \frac{1}{2}\Big\{[L_{t+\frac{1}{2}}]_{(i,i)} 
    + \sqrt{[L_{t+\frac{1}{2}}]_{(i,i)}^2 +4\gamma}\Big\} \; .
\]
Remark how this proximal operator behaves as a smooth projection on the set of strictly positive numbers. If the diagonal is negative after the gradient update, it brings  it to a small positive value. If  it was already positive, it slightly increases its value.

\subsection{Equality of Smoothness Constants}
\label{apdx:normal_constants}

In this section, we show that the Lipschitz smoothness parameter $\smoothness$, which appears in the convergence rate $\eqref{eq:mirror_prox_rate}$, is the same for both  causal and anticausal models.
Similarly to the categorical case, we reason about marginals loss first because they have a simpler form. 

The loss of a marginal mechanism is
$f_x(L,\zeta)=\inv{2} \sum_{i=1}^d (\zeta_i - L_i^T x)^2$
where $x$ is a sample observation and the $L_i$ are the columns of $L$.
We need to show that its Hessian is upper-bounded $\|\nabla^2 f_x(L,\zeta)\| \leq \smoothness$ .
Thanks to the objective $f_x$ being quadratic, the Hessian is independent of the parameters $(L,\zeta)$.
It depends only on the data $x$.
Since the data domain is \textit{a priori} the same for causal and anticausal models
-- e.g. $X$ and $Y$ can live in the same range --
the upper bound for the Hessian is the same.
This holds true at least for marginal mechanisms, because their loss is written exactly like above.

For conditional mechanisms, this is a bit more complicated but the reasoning holds.
The objective is similar, with extra parameters coming from the linear relationship between $X$ and $Y$. 
The Cholesky parametrization is described in equation~\eqref{eq:normal_cholesky_param}.
The conditional model uses $\zeta_{Y|X} = MX + m$ where $M$ is a matrix, $m$ is a vector and $X$ is a given sample.
This means that the objective is still a quadratic and that 2nd order derivatives w.r.t. $L_{Y|X}$ are still    independent of the parameters $M$ and $m$.
They depend only on the observed values of $X$ and $Y$.
We assume that these variables have the same domain \textit{a priori}, therefore both models have similar worst case smoothness constants.

\section{\uppercase{Normal Analysis}}
\label{apdx:normal_analysis}
In this section we introduce three different parametrization of the multivariate normal cause-effect model. 
The mean parametrization is the most common and intuitive, but it yields a non-convex optimization problem. 
The natural parametrization yields a convex problem with convergence guarantees, but it has no closed update formulas for our optimization algorithm of choice. 
The Cholesky parametrization offers both a convex problem and simple updates. 

Then we proceed to study how interventions induce distance in parameter space. We prove that in the natural parameter space, an intervention on the cause will create more distance in the anticausal model than in the causal model.

\subsection{Mean Parameters}
Cause $X$ and effect $Y$ are sampled from the causal model
\begin{align*}
    X &\sim \cN(\mu_X, \Sigma_X) \\
    Y|X &\sim \cN(A X + a, \Sigma_{Y|X})
\end{align*}
with parameters $(\mu_X, \Sigma_X, A, a, \Sigma_{Y|X})$.
All along, we will assume that all normal laws are non-degenerate -- e.g. $\Sigma_X > 0, \Sigma_{Y|X} >0$.
We  compute the marginal mean and covariance of $Y$ as well as the covariance between $X$ and $Y$
\begin{align*}
    \expect{Y} &= A\mu_X + a \\
    \Cov[Y] &= \Sigma_{Y|X} + A\Sigma_X A^\top \\
    \Cov[X,Y] &= \Sigma_X A^\top \; .
\end{align*}
From there we can derive the joint distribution as a function of the causal parameters
\begin{equation}
    \begin{pmatrix}
    X \\ Y
    \end{pmatrix}
    \sim \cN \left (
    \begin{pmatrix}
    \mu_X \\
    A\mu_X + a
    \end{pmatrix},
    \begin{pmatrix}
    \Sigma_X & \Sigma_X A^\top \\
    A\Sigma_X & \Sigma_{Y|X} + A\Sigma_X A^\top
    \end{pmatrix}
    \right ) \; .
\end{equation}

\subsection{Natural Parameters}
We want the negative log-likelihood objective to be convex, so we are going to use the natural parameters instead of the mean parameters
\begin{equation}
    \cN(\mu, \Sigma)
    = \natgauss(\eta = \Sigma^{-1}\mu,
        \Lambda = \Sigma^{-1})
\end{equation}
where we are using $\natgauss$ to explicit that this is taking the natural parameters as arguments
Our causal model using natural parameters is:
\begin{align}
\label{eq:normal_natural_model}
    X &\sim \natgauss(\eta_X, \Lambda_X) \\
    Y|X &\sim \natgauss(B X + b, \Lambda_{Y|X}) \nonumber
\end{align}
where we get the natural parameters from the mean parameters with formulas
\begin{align*}
    \Lambda_X &=\Sigma_X^{-1}\\
    \Lambda_{Y|X} &= \Sigma_{Y|X}^{-1}\\
    \eta_X &= \Lambda_X \mu_X\\
    B&= \Lambda_{Y|X} A\\
    b&= \Lambda_{Y|X} a
\end{align*}

\subsubsection{Switching Direction}
We want to get the natural parameters of the anticausal model as a function of the causal parameters.
To do so we are going to express the natural parameters of the joint, and then we will simply have to swap rows and columns to invert the roles of $X$ and $Y$.
To get the joint precision matrix, we need to invert the joint covariance. 
We use the Schur complement and the blockwise matrix inversion formulas
\begin{align*}
    & M = \begin{pmatrix}
    A & B \\
    C & D
    \end{pmatrix} \\
    & M/D := D - CA\iinv B \\
    & M\iinv =\\
    &\begin{pmatrix}
    A\iinv + A\iinv B (M/D)\iinv C A \iinv  & -A\iinv B (M/D)\iinv \\
    (M/D)\iinv C A \iinv  & (M/D)\iinv
    \end{pmatrix}
\end{align*}
In our case, the Schur complement of $\Sigma$ with respect to its lower right block $\Sigma_Y$ is precisely
\begin{align*}
    M / D
    & =\Sigma / \Sigma_Y \\
    &=  \Sigma_{Y|X} + A\Sigma_X A^\top - A\Sigma_X \Sigma_X\iinv \Sigma_X A^\top \\
    &= \Sigma_{Y|X} \; .
\end{align*}
By applying the formula and identifying the natural parameters, we get 
\begin{align*}
    \Lambda = \begin{pmatrix}
    \Lambda_X + B^\top \Lambda_{Y|X}^{-1}B & -B^\top \\
    -B & \Lambda_{Y|X}
    \end{pmatrix}
\end{align*}
To get the first natural parameter, all we have to do is to multiply the joint precision and the joint mean
\begin{align*}
    \eta 
    &= \Lambda \mu   \\
    &=\begin{pmatrix}
    \Lambda_X\mu_X + B^\top \Lambda_{Y|X}^{-1}B\mu_X - B^\top A \mu_X  - B^\top a \\
    -B \mu_X + \Lambda_{Y|X} A \mu_X + \Lambda_{Y|X} a
    \end{pmatrix} \\
    &=\begin{pmatrix}
    \eta_X - B^\top\Lambda_{Y|X}^{-1}b \\b
    \end{pmatrix}
\end{align*}
where $B^\top \Lambda_{Y|X}^{-1}B\mu_X - B^\top A \mu_X$, and $ -B \mu_X + \Lambda_{Y|X} A \mu_X $ are zero and we express the other terms with natural parameters.
Overall, the joint natural parameters are
\begin{align*}
    \natgauss \left (
    \begin{pmatrix}
    \eta_X - B^\top\Lambda_{Y|X}^{-1}b \\
    b
    \end{pmatrix},
    \begin{pmatrix}
    \Lambda_X + B^\top \Lambda_{Y|X}^{-1}B & -B^\top \\
    -B & \Lambda_{Y|X}
    \end{pmatrix}
    \right ) \; .
\end{align*}

From there we can use a symmetry argument to switch from causal $X \rightarrow Y$ to anticausal $X \leftarrow Y$ model.
\begin{align*}
    Y & \sim \natgauss(\eta_Y, \Lambda_Y) \\
    X|Y & \sim \natgauss(C Y + c, \Lambda_{X|Y})
\end{align*}
with the following formulas for the conditional mechanism
\begin{align}
    C &= B^\top \\
    c &= \eta_X - B^\top\Lambda_{Y|X}^{-1}b \\
    \Lambda_{X|Y} &=\Lambda_X + B^\top\Lambda_{Y|X}\iinv B
    \label{eq:normal_switch_conditional}
\end{align}
followed by these formulas for the marginal mechanisms
\begin{align}
    \Lambda_Y + C^\top\Lambda_{X|Y}\iinv C &= \Lambda_{Y|X} \\
    \eta_Y - C^\top \Lambda_{X|Y}\iinv c &= b \; .
\end{align}

These formulas are going to be very useful to establish a relationship between the distance to optimum of the causal and anticausal models in Appendix~\ref{apdx:normal_interventions}.

\subsection{Cholesky Parameters}

We call Cholesky parametrization of the normal law the parameters $(\mL,\zeta)$ such that 
\begin{align*}
     &\Lambda = \mL \mL^\top \quad \quad (\mL \; \textit{is lower triangular})\\
    &\zeta = \mL\iinv \eta = \mL^\top \mu
\end{align*}
We use $\chogauss(\zeta, \mL)$ to denote the normal law with Cholesky parameters $\zeta$ and $\mL$.
The full causal model~\eqref{eq:normal_natural_model} becomes
\begin{align}
    X &\sim \chogauss(\zeta_X, \mL_X) \nonumber\\
    Y|X &\sim \chogauss(M X + m, \mL_{Y|X})
\end{align}
where the 5 parameters are defined from the natural model by the equations
\begin{align}
    \label{eq:normal_cholesky_param}
    \mL_X L_X^\top &= \Lambda_X \nonumber\\
    \mL_{Y|X} \mL_{Y|X}^\top &= \Lambda_{Y|X}\nonumber \\
    \zeta_X &= \mL_X\iinv \eta_X\\
    M &= \mL_{Y|X}\iinv B\nonumber\\
    m &= \mL_{Y|X}\iinv b\nonumber
\end{align}
There is no closed formula to express the Cholesky decomposition of a sum of matrix $A+B$ with the Cholesky decomposition of $A$ and $B$. 
As a consequence, there is no simple formula to switch between causal and anticausal models with this parametrization.

\subsubsection{Joint Cholesky}
We derive a formula for the joint Cholesky for future reference.
To get a closed form for the joint Cholesky parameters from the conditional parameters, we need to switch the positions of $X$ and $Y$ in the joint vector -- e.g. we are using $(Y,X)$ instead of $(X,Y)$. Indeed the Cholesky decomposition is very dependent on the orders of the rows and columns. That's also why we cannot simply switch the column orders in the joint representation.
\begin{align*}
    \begin{pmatrix}
    Y \\ X
    \end{pmatrix}
    \sim \chogauss \left (
    \begin{pmatrix}
    m \\ \zeta_X
    \end{pmatrix},
    \begin{pmatrix}
    \mL_{Y|X} & 0 \\ -M^\top & \mL_X
    \end{pmatrix}
    \right ) \; .
\end{align*}
This joint representation is simply taking the conditional parameters and putting them in an array. It has the advantage that the distance is equal to the conditional distance.
It hints towards the idea that for multivariate normal variables, knowing the right Cholesky decomposition is equivalent to knowing the right causal graph.

\subsection{Kullback-Leibler Divergence}
We express the KL divergence in all three parametrizations because we use them in the code.
\begin{align*}
    &2\KL(\cN_0 || \cN_1)\\
    &= (\mu_1 - \mu_0)^\top \Sigma_1\iinv(\mu_1 - \mu_0) \\
    & \quad + \Tr(\Sigma_1\iinv \Sigma_0) - k
    - \log \abs{\Sigma_1\iinv \Sigma_0}\\
    &= \eta_1^\top \Lambda_1\iinv \eta_1 
    - 2 \eta_1^\top \Lambda_0\iinv \eta_0
    + \eta_0^\top  \Lambda_0\iinv \Lambda_1 \Lambda_0\iinv \eta_0 \\
    &\quad + \Tr(\Lambda_1 \Lambda_0\iinv) -k 
    - \log(\abs{\Lambda_1 \Lambda_0\iinv} \\
    &= \norm{V^\top \zeta_0 - \zeta_1}^2 
    + \norm{V}_F^2 -k - 2 \log\abs{V}
\end{align*}
where $V := \mL_0\iinv \mL_1$ is a lower triangular matrix which plays a special role. 

\subsection{Distance after Intervention}
\label{apdx:normal_interventions}
In this section we evaluate the effect on interventions on the cause and effect for both models. When the intervention happens on the cause, we replace $\mu_X, \Sigma_X$ by $\tilde \mu_X, \tilde \Sigma_X$, or equivalently we replace the natural parameters of the marginal on $X$.
The natural causal distance is simply
\begin{equation}
    \delta_\text{causal}
    = \|\eta_X - \tilde \eta_X \|^2 
    +  \|\Lambda_X - \tilde \Lambda_X \|^2_F \; .
\end{equation}
Unless indicated otherwise, we will consider the Frobenius distance between matrices.
For the anticausal model, both  marginal and conditional parameters need to change. Here similar to categorical case, we have 
\[
     \delta_\text{anticausal} \geq  \delta_\text{causal}. 
\]
However when the intervention happens on the effect, there is no clear formal relation between $\delta_\text{causal}$ and $\delta_\text{anticausal}$. More detail about the deriving the mathematical formula for $\delta_\text{causal}$ and $\delta_\text{anticausal}$ is presented in the following. 

\subsubsection{Intervention on Cause}
 When the intervention happens on the cause, the natural causal distance is 
\begin{equation}
    \delta_\text{causal}
    = \|\eta_X - \tilde \eta_X \|^2 
    +  \|\Lambda_X - \tilde \Lambda_X \|^2_F \; .
\end{equation}
How does this transformation affect the anticausal parameters? Both the marginal and the conditional have to adapt. The anticausal conditional is elegantly expressed with the causal natural parameters in \eqref{eq:normal_switch_conditional}, so we will start with the conditional 
\begin{align}
    C - \tilde C= B^\top  - \tilde B^\top = 0  \\
    c - \tilde c= \eta_X - \tilde \eta_X \\
    \Lambda_{X|Y} - \tilde \Lambda_{X|Y} = \Lambda_X - \tilde \Lambda_X \; .
\end{align}
In words, the linear transformation is invariant, the bias moves like the mean of $X$, and the conditional precision moves like the precision of $X$. This means that we can directly lower bound the anticausal distance with the causal distance
\begin{align}
    \delta_\text{anticausal}
    &= 
        \norm{C - \tilde C}^2
        + \norm{c - \tilde c}^2
        + \norm{\Lambda_{X|Y} - \tilde \Lambda_{X|Y}}^2
    \nonumber \\ 
    & + \norm{\Lambda_Y - \tilde \Lambda_Y}^2 
    + \norm{\eta_Y - \tilde \eta_Y} ^2 \nonumber \\
    &= \delta_\text{causal}
    + \norm{\Lambda_Y - \tilde \Lambda_Y}^2 
    + \norm{\eta_Y - \tilde \eta_Y} ^2\nonumber\\
    &> \delta_\text{causal} \; .
\end{align}
 We could get a stronger bound by bounding the anticausal marginal parameters, but any such bound would involve the value of the linearity $B$ and make the result needlessly more complicated.
 
 \subsubsection{Intervention on Effect}
We perform an intervention on $Y$ such that the causal model become independent
\begin{align*}
    X &\sim \cN(\mu_X, \Sigma_X) \\
    Y &\sim \cN(\tilde \mu_Y, \tilde \Sigma_{Y}) \; .
\end{align*}
This independence means that the linear models have a slope 0
\begin{equation}
\tilde A= \tilde B = \tilde C = 0 \; .
\end{equation}
The bias then has to account for the mean parameter
\begin{equation}
    \tilde a = \tilde \mu_Y,
    \quad \tilde b = \tilde \eta_Y,
    \quad \tilde c = \eta_X
\end{equation}
And the conditional precision have to match the marginal precision
\begin{equation}
    \tilde \Sigma_{Y|X} = \tilde \Sigma_Y ,
    \quad \tilde \Lambda_{Y|X} = \tilde \Lambda_Y,
    \quad \tilde \Lambda_{X|Y} = \Lambda_X
\end{equation}
So the distances are written
\begin{align*}
    \delta_\text{causal}
    &= \|B\|_F^2
    + \|b - \tilde \eta_Y\|^2 
    +  \|\Lambda_{Y|X} - \tilde \Lambda_Y \|^2_F \\
    \delta_\text{anticausal}
    &= \|C\|_F^2
    + \|c - \eta_X \|^2
    + \|\Lambda_{X|Y} - \Lambda_X \|^2_F\\
    &+ \| \eta_Y - \tilde \eta_Y \|^2
    + \| \Lambda_Y - \tilde \Lambda_Y \|_F^ 2 \\
    &= \|B\|_F^2
    + \|B^\top\Lambda_{Y|X}^{-1}b \|^2
    + \|B^\top\Lambda_{Y|X}^{-1}B\|^2_F\\
    &+ \| \eta_Y - \tilde \eta_Y \|^2
    + \| \Lambda_Y - \tilde \Lambda_Y \|_F^ 2
\end{align*}

We did not find any meaningful simplification of these formulas. 
 
\section{\uppercase{Normal prior}}
\label{apdx:normal_experiments}
 
Exactly like in the categorical setting, the distributions we sample are going to impact the speed of adaptation and the distances we measure. Let $K$ be the dimension of $X$ and $Y$, and $n_0=2K+2$ an arbitrary number of prior observations. 
We define a \emph{pseudo-conjugate prior}
\begin{align}
    \Lambda_X  
    &\sim \cW(n_0, \frac{I_K}{K}) \\
    \eta_X |\Lambda_{X} 
    &\sim \cN(0, \frac{\Lambda_X}{n_0}) \\
    \Lambda_{Y|X} 
    &\sim \cW(n_0, \textcolor{red}{10}\frac{I_K}{K})\\
    b | \Lambda_{Y|X}
    &\sim \cN(0, \frac{\Lambda_{Y|X}}{n_0}) \\
    B &=\Lambda_{Y|X} A \text{ where } A \sim \cN(0, \frac{I_{K^2}}{\textcolor{red}{\sqrt{K}}})
    \label{eq:normal_prior}
\end{align}
where $\cW$ is the Wishart distribution with parameters: degrees of freedom and scale matrix.
We picked these parameters such that $\eta_Y,\Lambda_Y$ follows approximately the same law as $\eta_X,\Lambda_X$.
Two important factors to get  a symmetric relationship between $X$ and $Y$ are \textcolor{red}{10 and $\sqrt{K}$}.
First, we sample a larger conditional precision, so that their relationship is quite deterministic. 
Second we sample the linear layer such that it preserves the scale of $X$, so that $X$ and $Y$ have approximately the same variance. 
We also use appropriate covariance matrices to sample other parameters such that the prior is somewhat conjugate and gives proper variance formulas. 
 
 We sample interventions from the same distributions as the cause marginal in \eqref{eq:normal_prior}. For an intervention on the cause
 \begin{align*}
     \tilde \Lambda_X  &\sim \cW(n_0, \frac{I_K}{K}) \\
     \tilde \eta_X |\tilde \Lambda_X &\sim \cN(0, \frac{\tilde \Lambda_X}{n_0}) \; .
 \end{align*}
 For an intervention on the effect 
 \begin{align*}
     \tilde B = 0
     \tilde \Lambda_{Y|X} = \tilde \Lambda_{Y} & \sim \cW(n_0, \frac{I_K}{K}) \\
     \tilde b = \tilde \eta_Y |\tilde \Lambda_{Y} &\sim \cN(0, \frac{\tilde \Lambda_{Y}}{n_0}) \; .
 \end{align*}

\end{document}